\documentclass{article}
\usepackage[utf8]{inputenc}
\usepackage{amsmath,amssymb,amsthm}
\usepackage{mathtools}
\usepackage{booktabs}
\usepackage{hyperref}
\usepackage{cleveref}
\usepackage[margin=1in]{geometry}
\usepackage{natbib}
\usepackage{algorithm}
\usepackage{algorithmic}
\usepackage{graphicx}

\newtheorem{theorem}{Theorem}
\newtheorem{proposition}{Proposition}
\newtheorem{corollary}{Corollary}

\newtheorem{remark}{Remark}
\newtheorem{definition}{Definition}

\newcommand{\R}{\mathbb{R}}
\newcommand{\SO}{\mathrm{SO}}
\newcommand{\tr}{\mathrm{tr}}
\newcommand{\diag}{\mathrm{diag}}

\newcommand{\Id}{\mathbf{I}}
\newcommand{\norm}[1]{\left\lVert #1 \right\rVert}
\newcommand{\frob}[1]{\left\lVert #1 \right\rVert_F}
\newcommand{\abs}[1]{\left\lvert #1 \right\rvert}
\newcommand{\pderiv}[2]{\frac{\partial #1}{\partial #2}}

\title{Training Without Orthogonalization, Inference With SVD:\\ A Gradient Analysis of Rotation Representations}

\author{
  Chris Choy\\
  NVIDIA\\
  \texttt{cchoy@nvidia.com}
}

\date{}

\begin{document}
\maketitle

\begin{abstract}
Recent work has shown that removing orthogonalization during training and applying it only at inference improves rotation estimation in deep learning, with empirical evidence favoring 9D representations with SVD projection~\citep{gu2024prom}.
However, the theoretical understanding of \emph{why} SVD orthogonalization specifically harms training, and why it should be preferred over Gram-Schmidt at inference, remains incomplete.
We provide a detailed gradient analysis of SVD orthogonalization specialized to $3 \times 3$ matrices and $\SO(3)$ projection.
Our central result derives the exact spectrum of the SVD backward pass Jacobian: it has rank $3$ (matching the dimension of $\SO(3)$) with nonzero singular values $2/(s_i + s_j)$ and condition number $\kappa = (s_1 + s_2)/(s_2 + s_3)$, creating quantifiable gradient distortion that is most severe when the predicted matrix is far from $\SO(3)$ (e.g., early in training when $s_3 \approx 0$).
We further show that even stabilized SVD gradients~\citep{wang2021robust} introduce gradient direction error, whereas removing SVD from the training loop avoids this tradeoff entirely.
We also prove that the 6D Gram-Schmidt Jacobian has an asymmetric spectrum: its parameters receive unequal gradient signal, explaining why 9D parameterization is preferable.
Together, these results provide the theoretical foundation for training with direct 9D regression and applying SVD projection only at inference.
\end{abstract}

\section{Introduction}
\label{sec:intro}

Representing 3D rotations for deep learning is a fundamental problem in computer vision and robotics. A neural network must output a rotation $R \in \SO(3)$, but the rotation group is a 3-dimensional manifold embedded in $\R^{3 \times 3}$ with the orthogonality constraint $R^\top R = \Id, \det(R) = 1$.
This has led to diverse representations: axis-angle, quaternions, 6D~\citep{zhou2019continuity}, and 9D with SVD projection~\citep{levinson2020analysis}.

Two recent lines of work have reached different conclusions about when orthogonalization should be applied.
\citet{levinson2020analysis} showed that SVD orthogonalization is the maximum likelihood estimator under isotropic Gaussian noise and, to first order, produces half the expected reconstruction error of Gram-Schmidt, advocating for SVD as a differentiable layer during both training and inference.  We call this \textbf{SVD-Train} (following the terminology of \citet{levinson2020analysis}).  We use \textbf{SVD-Inference} to denote the alternative where SVD is applied only at inference, with the training loss computed on the raw (unorthogonalized) matrix.
\citet{gu2024prom} showed that \emph{any} orthogonalization during training introduces gradient pathologies (ambiguous updates, exploding gradients, suboptimal convergence) and proposed learning ``pseudo rotation matrices'' (PRoM), which is an instance of SVD-Inference: direct 9D regression during training with SVD projection applied only at test time.

However, key questions remain unanswered. The PRoM analysis treats all orthogonalizations uniformly through a general non-injectivity argument, without dissecting the \emph{specific} gradient failure modes of SVD versus Gram-Schmidt. While their ablation study shows that SVD at inference outperforms Gram-Schmidt (54.8 vs.\ 55.6 PA-MPJPE), no theoretical justification is provided for this choice. Moreover, no prior work has formally analyzed why 9D regression is preferable to 6D regression in the unorthogonalized training regime.

We address these questions through a detailed Jacobian analysis of the SVD and Gram-Schmidt orthogonalization mappings.  Our primary contribution is a \textbf{detailed analysis of SVD gradient pathology} specialized to $3 \times 3$ matrices and $\SO(3)$ projection:
\begin{itemize}
    \item We derive the exact spectrum of the SVD Jacobian: it has rank $3$ with nonzero singular values $2/(s_i + s_j)$ for each pair $i < j$, giving spectral norm $2/(s_2+s_3) = O(1/\delta)$ and condition number $\kappa = (s_1 + s_2)/(s_2 + s_3)$ (\Cref{thm:svd_gradient_bound}).  This significantly extends PRoM's qualitative observation about the $K$ matrix by providing the exact spectral characterization.
    \item We show that even state-of-the-art stabilization~\citep{wang2021robust} cannot eliminate this pathology without introducing gradient direction error, making avoidance strictly preferable to mitigation (\Cref{prop:taylor_comparison}).
    \item We quantify the \emph{gradient information loss}: SVD backpropagation retains only $1/3$ of gradient energy, discarding the 6-dimensional normal component that encodes distance from $\SO(3)$ (\Cref{prop:info_loss}).
\end{itemize}

We also prove that Gram-Schmidt's Jacobian has an asymmetric spectrum (\Cref{thm:gs_asymmetry}), explaining why 9D is preferable to 6D.  Combined with SVD's optimality as an inference-time projector ($3\times$ error reduction, \Cref{cor:factor_three}), these results explain why 9D + SVD-inference works.

\begin{figure*}[t]
    \centering
    \includegraphics[width=\textwidth]{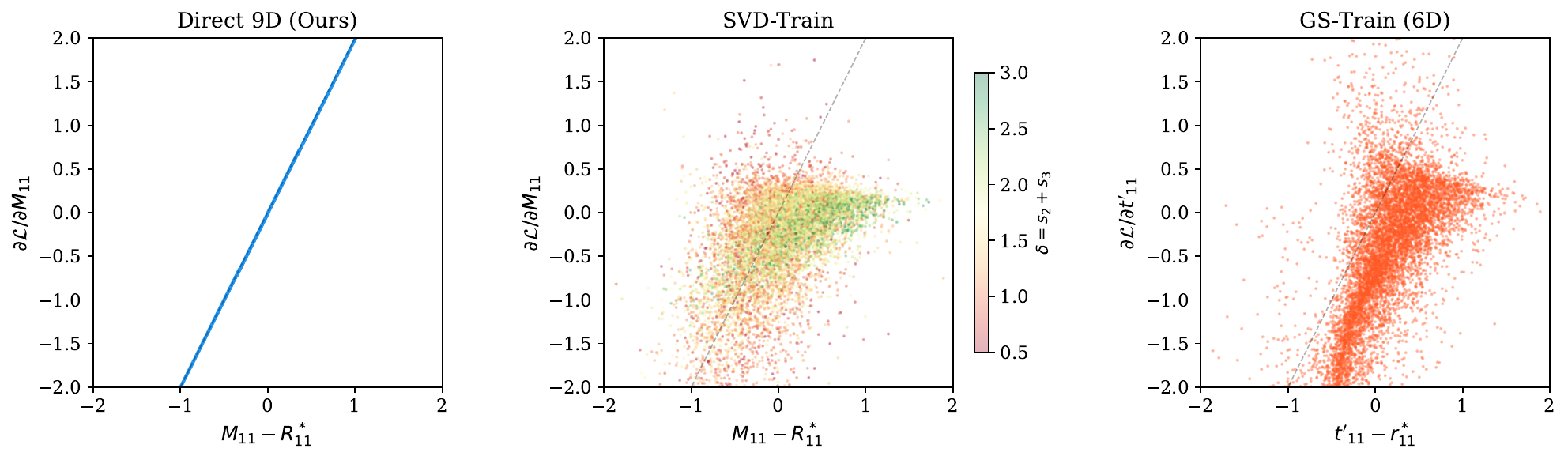}
    \caption{Gradient of $\mathcal{L}$ w.r.t.\ $M_{11}$ (or $t'_{11}$) vs.\ the error $M_{11} - R^*_{11}$, with $\sigma = 0.5$ Gaussian noise ($10{,}000$ samples). \textbf{Left:} Direct 9D gradients lie on the diagonal (each element depends only on its own error). \textbf{Center:} SVD-Train gradients scatter across all quadrants; color encodes the singular value gap $\delta = s_2 + s_3$ (small $\delta$ in red = most erratic). \textbf{Right:} GS-Train (6D) also produces ambiguous gradients from cross-column coupling.}
    \label{fig:gradient_scatter}
\end{figure*}

\section{Related Work}
\label{sec:related}

\paragraph{Rotation representations.}
Classical parameterizations (Euler angles, axis-angle, quaternions) are all discontinuous mappings from $\SO(3)$ to $\R^d$ for $d \leq 4$, a topological necessity proved by \citet{stuelpnagel1964parameterization}.
\citet{zhou2019continuity} proposed the first continuous representation for neural networks: a 6D parameterization using two columns of the rotation matrix, recovered via Gram-Schmidt orthogonalization.
\citet{levinson2020analysis} subsequently advocated for a 9D representation with SVD orthogonalization, showing it is the maximum likelihood estimator under Gaussian noise and produces half the expected error of Gram-Schmidt.
Alternative approaches include spherical regression on $n$-spheres~\citep{liao2019spherical}, smooth quaternion representations~\citep{peretroukhin2020smooth}, mixed classification-regression frameworks~\citep{mahendran2018mixed}, probabilistic models using von Mises~\citep{prokudin2018deep} or matrix Fisher distributions~\citep{gilitschenski2020deep}, and direct manifold regression~\citep{bregier2021deep}. Both \citet{gilitschenski2020deep} and \citet{bregier2021deep} use unconstrained parameterizations, avoiding orthogonalization constraints during training.
\citet{geist2024learning} provide a comprehensive survey of rotation representations, empirically observing that SVD gradient ratios between columns stay closer to 1 than GS (their Fig.~7), and framing SVD as an ``ensemble'' where all columns contribute equally. Our work provides the rigorous mathematical foundations for these empirical observations.
We note that \citet{geist2024learning} report that directly predicting rotation matrix entries is worse than using SVD/GS layers; however, their comparison uses orthogonalization-based losses during training, not the MSE-on-raw-matrix approach of PRoM~\citep{gu2024prom}, which is the key distinction.

\paragraph{Orthogonalization during training.}
\citet{gu2024prom} showed that incorporating any orthogonalization (Gram-Schmidt or SVD) during training introduces gradient pathologies: ambiguous updates, exploding gradients, and provably suboptimal convergence.
They proposed Pseudo Rotation Matrices (PRoM), which remove orthogonalization during training and apply it only at inference, achieving state-of-the-art results on human pose estimation benchmarks.
PRoM's theoretical contributions are two-fold: (i) a convergence bound showing that the loss gap $L(\hat{R}_i) - L(R^*)$ is controlled by $\psi(B_i) = 1/\sqrt{\lambda_{\min}(B_i B_i^\top)}$, and (ii) a proof that $\psi(B_1) = \infty$ for any non-locally-injective orthogonalization $g$, while $\psi(B_0) = 1$ when $g$ is removed.
These results are general: they treat Gram-Schmidt and SVD identically as instances of non-injective $g$, without deriving explicit bounds on SVD gradient magnitude, conditioning, or directional distortion.
PRoM predicts all 9 matrix elements during training (not 6), and their ablation shows SVD at inference outperforms Gram-Schmidt (54.8 vs.\ 55.6 PA-MPJPE). But no theoretical justification is given for why 9D outperforms 6D, or why SVD beats Gram-Schmidt at inference.

We build on PRoM's insight by providing the theory: SVD-specific gradient analysis with explicit Jacobian bounds, a characterization of Gram-Schmidt's asymmetric gradient structure, and a justification for preferring SVD over Gram-Schmidt at inference.

\paragraph{SVD gradient stability.}
Differentiating through SVD is known to be numerically challenging~\citep{ionescu2015matrix,giles2008matrix,townsend2016differentiating}.
The backward pass involves a kernel matrix with entries $1/(s_i^2 - s_j^2)$ that diverge when singular values coincide~\citep{ionescu2015matrix}.
\citet{wang2021robust} proposed a Taylor expansion approximation to stabilize SVD gradients, bounding the gradient scaling factor by $n(K+1)/\varepsilon$ (where $K$ is the expansion degree and $\varepsilon$ is a clamping threshold), but at the cost of gradient direction error (up to $5.71°$ in the worst case when the dominant eigenvalue covers at least $50\%$ of the energy).
Our theoretical analysis shows that this tradeoff is unnecessary: removing SVD from the training loop, as \citet{gu2024prom} proposed, yields exact gradients with no direction error and no hyperparameters.

\section{Preliminaries}
\label{sec:prelim}

\subsection{Rotation Representations for Deep Learning}

A neural network $f_\mathbf{w}: \mathcal{X} \to \R^d$ maps input $\mathbf{x}$ to a $d$-dimensional output, which is then mapped to $\SO(3)$ via a representation function $r: \R^d \to \R^{3 \times 3}$ and an orthogonalization function $g: \R^{3 \times 3} \to \SO(3)$. The predicted rotation is $\hat{R} = g(r(f_\mathbf{w}(\mathbf{x})))$.
We briefly review the main rotation representations used in the literature; for a comprehensive treatment see \citet{geist2024learning}.

\paragraph{Euler angles (3D).}
The oldest parameterization decomposes a rotation into three successive rotations about coordinate axes: $R(\alpha,\beta,\gamma) = R_3(\gamma)R_2(\beta)R_1(\alpha)$~\citep{euler1765mouvement}. Euler angles suffer from gimbal lock (singularities at $\beta = \pm\pi/2$), non-unique representations, and a discontinuous inverse map $g$, making them unsuitable for gradient-based learning~\citep{zhou2019continuity,geist2024learning}.

\paragraph{Exponential coordinates / rotation vectors (3D).}
A rotation is encoded as $\boldsymbol{\omega} \in \R^3$, where the direction gives the rotation axis and the norm $\|\boldsymbol{\omega}\|$ gives the angle. The rotation matrix is recovered via the matrix exponential of the skew-symmetric form $[\boldsymbol{\omega}]_\times$, or equivalently Rodrigues' formula~\citep{grassia1998practical}. Because $\boldsymbol{\omega}$ and $(\|\boldsymbol{\omega}\| - 2\pi)\boldsymbol{\omega}/\|\boldsymbol{\omega}\|$ encode the same rotation (double cover), the inverse map $g$ is discontinuous~\citep{stuelpnagel1964parameterization}.

\paragraph{Quaternions (4D).}
Unit quaternions $q = (w,x,y,z) \in \mathcal{S}^3$ provide a smooth, singularity-free parameterization related to axis-angle by $w = \cos(\alpha/2)$ and $(x,y,z) = \sin(\alpha/2)\tilde{\boldsymbol{\omega}}$~\citep{grassia1998practical}. However, unit quaternions double-cover $\SO(3)$ ($q$ and $-q$ represent the same rotation), so any continuous inverse map $g: \SO(3) \to \mathcal{S}^3$ is impossible~\citep{stuelpnagel1964parameterization}. Augmented quaternion losses and smooth parameterizations have been proposed to mitigate this~\citep{peretroukhin2020smooth}.

\paragraph{6D representation with Gram-Schmidt.}
The network outputs $\mathbf{t}_1', \mathbf{t}_2' \in \R^3$ (6 parameters). The orthogonalization $g_{\mathrm{GS}}$ produces:
\begin{equation}\label{eq:gs}
    \mathbf{r}_1' = \frac{\mathbf{t}_1'}{\norm{\mathbf{t}_1'}}, \quad
    \mathbf{r}_2' = \frac{\mathbf{t}_2' - (\mathbf{r}_1' \cdot \mathbf{t}_2')\mathbf{r}_1'}{\norm{\mathbf{t}_2' - (\mathbf{r}_1' \cdot \mathbf{t}_2')\mathbf{r}_1'}}, \quad
    \mathbf{r}_3' = \mathbf{r}_1' \times \mathbf{r}_2'.
\end{equation}

\paragraph{9D representation with SVD.}
The network outputs all 9 elements of a matrix $M \in \R^{3 \times 3}$. The special orthogonalization~\citep{levinson2020analysis} is:
\begin{equation}\label{eq:svd}
    g_{\mathrm{SVD}}(M) = \mathrm{SVDO}^+(M) := U \Sigma' V^\top, \quad \Sigma' = \diag(1, 1, \det(UV^\top)),
\end{equation}
where $M = U\Sigma V^\top$ is the SVD of $M$. This is the closest rotation matrix to $M$ in Frobenius norm~\citep{arunprocrustes}:
\begin{equation}\label{eq:svd_optimality}
    \mathrm{SVDO}^+(M) = \arg\min_{R \in \SO(3)} \frob{R - M}^2.
\end{equation}

\subsection{The SVD Backward Pass}

For a loss $\mathcal{L}(M, R) = \frob{\mathrm{SVDO}^+(M) - R}^2$, the gradient $\pderiv{\mathcal{L}}{M}$ requires differentiating through the SVD. Building on the general SVD Jacobian framework of \citet{papadopoulo2000estimating}, we specialize the derivation to the rotation projection $M \mapsto R = UV^\top$ for $3 \times 3$ matrices and derive the complete spectrum of the resulting Jacobian.

Let $M = U \Sigma V^\top$ with singular values $s_1 \geq s_2 \geq s_3 \geq 0$, and let $R = UV^\top$.  Consider a perturbation $\mathrm{d}M$.  The orthogonality constraints $U^\top U = \Id$ and $V^\top V = \Id$ imply that $U^\top \mathrm{d}U$ and $V^\top \mathrm{d}V$ are antisymmetric.  Differentiating $M = U \Sigma V^\top$ and left-multiplying by $U^\top$, right-multiplying by $V$, we obtain:
\begin{equation}\label{eq:svd_diff}
    U^\top \mathrm{d}M \, V = (U^\top \mathrm{d}U)\Sigma + \mathrm{d}\Sigma + \Sigma(V^\top \mathrm{d}V)^\top.
\end{equation}
The diagonal of \eqref{eq:svd_diff} gives $\mathrm{d}\Sigma_{ii} = (U^\top \mathrm{d}M \, V)_{ii}$.  Letting $P = U^\top \mathrm{d}M\, V$, $A = U^\top \mathrm{d}U$ (antisymmetric), and $\Omega = V^\top \mathrm{d}V$ (antisymmetric), the off-diagonal entries ($i \neq j$) of \eqref{eq:svd_diff} yield $P_{ij} = s_j A_{ij} - s_i \Omega_{ij}$ and $P_{ji} = -s_i A_{ij} + s_j \Omega_{ij}$, i.e., a $2 \times 2$ linear system:
\begin{equation}\label{eq:svd_system}
    \begin{pmatrix} s_j & -s_i \\ -s_i & s_j \end{pmatrix}
    \begin{pmatrix} A_{ij} \\ \Omega_{ij} \end{pmatrix}
    =
    \begin{pmatrix} P_{ij} \\ P_{ji} \end{pmatrix}.
\end{equation}
The determinant of this system is $s_j^2 - s_i^2$.  When $s_i \neq s_j$, solving gives~\citep{papadopoulo2000estimating,giles2008matrix}:
\begin{equation}\label{eq:omega_full}
    A_{ij} = \frac{s_j P_{ij} - s_i P_{ji}}{s_j^2 - s_i^2}, \quad
    \Omega_{ij} = \frac{s_j P_{ji} - s_i P_{ij}}{s_j^2 - s_i^2}.
\end{equation}
The differential of the rotation $R = UV^\top$ is $\mathrm{d}R = U \Phi V^\top$ where $\Phi = A - \Omega$ is antisymmetric.  Substituting \eqref{eq:omega_full}:
\begin{equation}\label{eq:phi_full}
    \Phi_{ij} = A_{ij} - \Omega_{ij} = \frac{(s_j + s_i)(P_{ij} - P_{ji})}{s_j^2 - s_i^2} = \frac{P_{ij} - P_{ji}}{s_j - s_i} = \frac{P_{ji} - P_{ij}}{s_i + s_j}, \quad i \neq j, \qquad \Phi_{ii} = 0,
\end{equation}
where the last step uses $s_j^2 - s_i^2 = (s_j - s_i)(s_j + s_i)$ and $s_j - s_i = -(s_i - s_j)$.
Each off-diagonal entry of $\Phi$ is a linear function of $\mathrm{d}M$, divided by $s_i + s_j$.  The resulting gradient for a loss $\mathcal{L}$ through $R = UV^\top$ can be written as~\citep{levinson2020analysis}:
\begin{equation}\label{eq:svd_grad}
    \pderiv{\mathcal{L}}{M} = U Z V^\top, \quad
    Z_{ij} = \begin{cases}
        \dfrac{-X_{ij}}{s_i + s_j}, & i \neq j, \\[6pt]
        0, & i = j,
    \end{cases}
\end{equation}
where $X = U^\top \pderiv{\mathcal{L}}{U} - \pderiv{\mathcal{L}}{U}^\top U$ is antisymmetric. For $\mathrm{SVDO}^+$ when $\det(M) < 0$, the factor $\Sigma' = \diag(1,1,-1)$ effectively replaces $s_3$ with $-s_3$, changing the denominator for pairs involving the third singular value to $s_i - s_3$ (for $i \in \{1,2\}$).

The derivation above makes explicit that the source of all SVD gradient pathologies is the $2 \times 2$ system \eqref{eq:svd_system}: its determinant $s_i^2 - s_j^2$ controls the sensitivity of singular vectors to input perturbations, and its reciprocal $1/(s_i + s_j)$ directly scales every gradient component in \eqref{eq:phi_full}.

\section{SVD Gradient Pathology: The Convergence Paradox}
\label{sec:svd_analysis}

The $1/(s_i + s_j)$ scaling in \eqref{eq:phi_full} creates three pathologies for training: gradient explosion, poor conditioning, and gradient coupling.

\subsection{Gradient Explosion and Conditioning}

Consider the mapping $g_{\mathrm{SVD}}: M \mapsto R = UV^\top$ where $M = U \Sigma V^\top$. We analyze the Jacobian $J_{\mathrm{SVD}} = \pderiv{\mathrm{vec}(R)}{\mathrm{vec}(M)} \in \R^{9 \times 9}$.

\begin{definition}[Singular value gap]
For a matrix $M$ with singular values $s_1 \geq s_2 \geq s_3 \geq 0$, define the minimum singular value gap as
\begin{equation}
    \delta(M) := \min_{i \neq j} |s_i + s_j|.
\end{equation}
For $\mathrm{SVDO}^+$ with $\det(M) > 0$, this equals $\min_{i \neq j}(s_i + s_j) = s_2 + s_3$. For $\det(M) < 0$ (where $s_3$ becomes $-s_3$), this equals $\min(s_1 - s_3, s_2 - s_3, s_1 + s_2)$.
\end{definition}

\begin{theorem}[SVD Jacobian spectrum]\label{thm:svd_gradient_bound}
Let $M \in \R^{3 \times 3}$ with $\det(M) > 0$ and SVD $M = U \Sigma V^\top$ with distinct singular values $s_1 > s_2 > s_3 > 0$. The Jacobian $J_{\mathrm{SVD}} = \pderiv{\mathrm{vec}(R)}{\mathrm{vec}(M)}$ of the mapping $M \mapsto R = UV^\top$ has rank $3$ (matching the dimension of $\SO(3)$) with a $6$-dimensional null space.  Its three nonzero singular values are:
\begin{equation}\label{eq:svd_jacobian_spectrum}
    \sigma(J_{\mathrm{SVD}}) = \left\{\frac{2}{s_1+s_2},\; \frac{2}{s_1+s_3},\; \frac{2}{s_2+s_3},\; 0,\; 0,\; 0,\; 0,\; 0,\; 0\right\}.
\end{equation}
Consequently, the spectral norm and condition number (of the restriction to the column space) are:
\begin{equation}\label{eq:svd_jacobian_bound}
    \norm{J_{\mathrm{SVD}}}_2 = \frac{2}{s_2 + s_3}, \qquad
    \kappa(J_{\mathrm{SVD}}) = \frac{s_1 + s_2}{s_2 + s_3}.
\end{equation}
When $s_3 \to 0$ (common early in training), $\norm{J_{\mathrm{SVD}}}_2 = O(1/s_3)$.
\end{theorem}

\begin{proof}
From \eqref{eq:phi_full}, $\mathrm{d}R = U\Phi V^\top$ where $\Phi_{ij} = (P_{ji} - P_{ij})/(s_i + s_j)$ and $P = U^\top \mathrm{d}M\, V$.  Since $U, V$ are orthogonal, $\frob{\mathrm{d}R} = \frob{\Phi}$ and $\frob{\mathrm{d}M} = \frob{P}$, so the singular values of $J_{\mathrm{SVD}}$ equal those of the linear map $\mathcal{L}: P \mapsto \Phi$.

The $9$ entries of $P$ decompose into orthogonal subspaces:
\begin{enumerate}
    \item \textbf{Diagonal entries} $P_{ii}$ ($3$ dimensions): $\Phi_{ii} = 0$, so these lie in the null space.
    \item \textbf{Symmetric off-diagonal combinations} $(P_{ij} + P_{ji})/\sqrt{2}$ for $i < j$ ($3$ dimensions): since $\Phi_{ij}$ depends on $P_{ji} - P_{ij}$, the symmetric combination maps to zero.  These also lie in the null space.
    \item \textbf{Antisymmetric off-diagonal combinations} $(P_{ji} - P_{ij})/\sqrt{2}$ for $i < j$ ($3$ dimensions): let $\alpha_{ij} = (P_{ji} - P_{ij})/\sqrt{2}$.  Then $\Phi_{ij} = \sqrt{2}\,\alpha_{ij}/(s_i + s_j)$ and $\Phi_{ji} = -\Phi_{ij}$, giving $\frob{\Phi}^2 = 2\Phi_{ij}^2 = 4\alpha_{ij}^2/(s_i+s_j)^2$.
\end{enumerate}
The null space has dimension $3 + 3 = 6$, confirming rank $3$.  For each pair $(i,j)$ with $i < j$, a unit antisymmetric input $\alpha_{ij} = 1$ produces $\frob{\mathrm{d}R} = 2/(s_i + s_j)$.  Since the three antisymmetric subspaces are orthogonal and map to orthogonal outputs (distinct entries of the antisymmetric matrix $\Phi$), the three nonzero singular values of $J_{\mathrm{SVD}}$ are exactly $\{2/(s_i+s_j)\}_{i<j}$.
\end{proof}

\begin{corollary}[Universality: $\det(M) < 0$]\label{cor:det_negative}
When $\det(M) < 0$, $\mathrm{SVDO}^+(M) = U\diag(1,1,-1)V^\top$.  The sign flip changes \emph{which} input subspace drives the output (symmetric off-diagonal for pairs involving index $3$, instead of antisymmetric), but the Jacobian spectrum is \textbf{identical} to the $\det(M) > 0$ case: $\sigma(J) = \{2/(s_1+s_2),\, 2/(s_1+s_3),\, 2/(s_2+s_3),\, 0,\ldots,0\}$.  The apparent $1/(s_i - s_3)$ singularity in the backward pass formula (obtained by absorbing $\diag(1,1,-1)$ into the gradient matrix) is a coordinate artifact: the numerator vanishes proportionally, so no additional divergence occurs.  See \Cref{app:det_negative} for the full proof.
\end{corollary}

\begin{remark}[Gradient instability during training]\label{rem:instability}
The gradient pathology identified in \Cref{thm:svd_gradient_bound} is most severe early in training when the network output $M$ is far from $\SO(3)$ and $s_3$ may be near zero, giving $O(1/s_3)$ gradient amplification.  As training progresses and singular values approach 1, the denominators $s_i + s_j \approx 2$ and the gradients stabilize, consistent with \citet{levinson2020analysis}'s empirical observation that SVD-Train gradient norms remain comparable to GS-Train.

However, even in the well-conditioned regime, the condition number $\kappa(J_{\mathrm{SVD}}) \geq (s_1 + s_2)/(s_2 + s_3)$ creates \emph{anisotropic} gradient scaling: perturbations in the $(2,3)$ singular vector plane are amplified relative to the $(1,2)$ plane. This anisotropy conflicts with the isotropic step sizes of standard optimizers (SGD, Adam), and any transient perturbation that drives $s_3$ toward zero (mini-batch noise, learning rate) can trigger temporary gradient explosion.
\end{remark}

\begin{remark}[Singular value switching]
When two singular values cross (e.g., $s_1$ and $s_2$ exchange ordering), the corresponding singular vectors swap, causing a discrete jump in $R = UV^\top$. At the crossing point $s_i = s_j$, the determinant $s_j^2 - s_i^2$ in \eqref{eq:svd_system} vanishes, and the gradient is undefined. Near the crossing, gradients exhibit discontinuous behavior. During training, singular values fluctuate continuously, making such crossings inevitable.
\end{remark}

\subsection{Even Stabilized SVD Gradients Are Suboptimal}

One might ask whether the SVD gradient pathology can be ``fixed'' rather than avoided. \citet{wang2021robust} proposed exactly this: replace the unstable kernel $1/(s_i^2 - s_j^2)$ with a $K$-th degree Taylor expansion and clamp singular values above a threshold $\varepsilon > 0$.

\begin{remark}[Comparison with stabilized SVD gradients]\label{prop:taylor_comparison}
\citet{wang2021robust} bound the Taylor-stabilized SVD gradient scaling factor by $n(K+1)/\varepsilon$ (independent of the singular value gap), but report worst-case direction error of $5.71°$ with $K = 9$ (when the dominant eigenvalue covers $\geq 50\%$ of energy), and $50\%$ training failure with $K=19$, $\varepsilon = 10^{-4}$.

We note that this comparison is structurally asymmetric: SVD-Train and direct regression optimize \emph{different} loss functions ($\frob{\mathrm{SVDO}^+(M) - R^*}^2$ vs.\ $\frob{M - R^*}^2$), so ``direction error'' is measured relative to each method's own objective.  Nevertheless, the comparison highlights that \emph{even within the SVD-Train framework}, stabilization introduces an unavoidable magnitude-vs-direction tradeoff with two hyperparameters.  Direct regression avoids this tradeoff entirely: its gradient $2(M - R^*)$ is exact, with $\kappa = 1$ and no hyperparameters.
\end{remark}

\begin{figure*}[t]
    \centering
    \includegraphics[width=\textwidth]{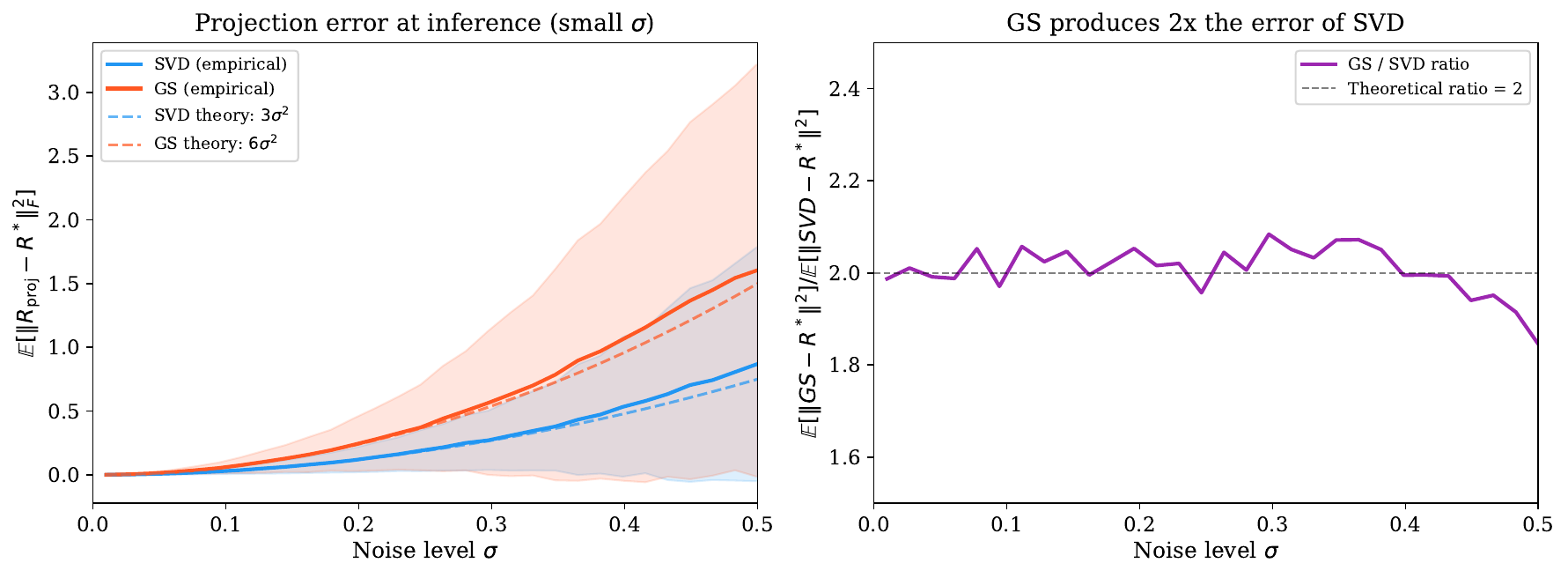}
    \caption{SVD vs.\ Gram-Schmidt projection error at inference, restricted to the small-$\sigma$ regime where the first-order theory of \citet{levinson2020analysis} is accurate (and where trained networks operate). \textbf{Left:} Expected squared Frobenius error $\mathbb{E}[\frob{R_{\mathrm{proj}} - R^*}^2]$ with $M = R^* + \sigma N$. Empirical results (solid, $5{,}000$ samples per $\sigma$) closely match the first-order predictions (dashed): $3\sigma^2 + O(\sigma^3)$ for SVD and $6\sigma^2 + O(\sigma^3)$ for GS. \textbf{Right:} The ratio of GS to SVD error stays near the theoretical value of 2 across the entire range.}
    \label{fig:projection_error}
\end{figure*}

\begin{figure*}[t]
    \centering
    \includegraphics[width=\textwidth]{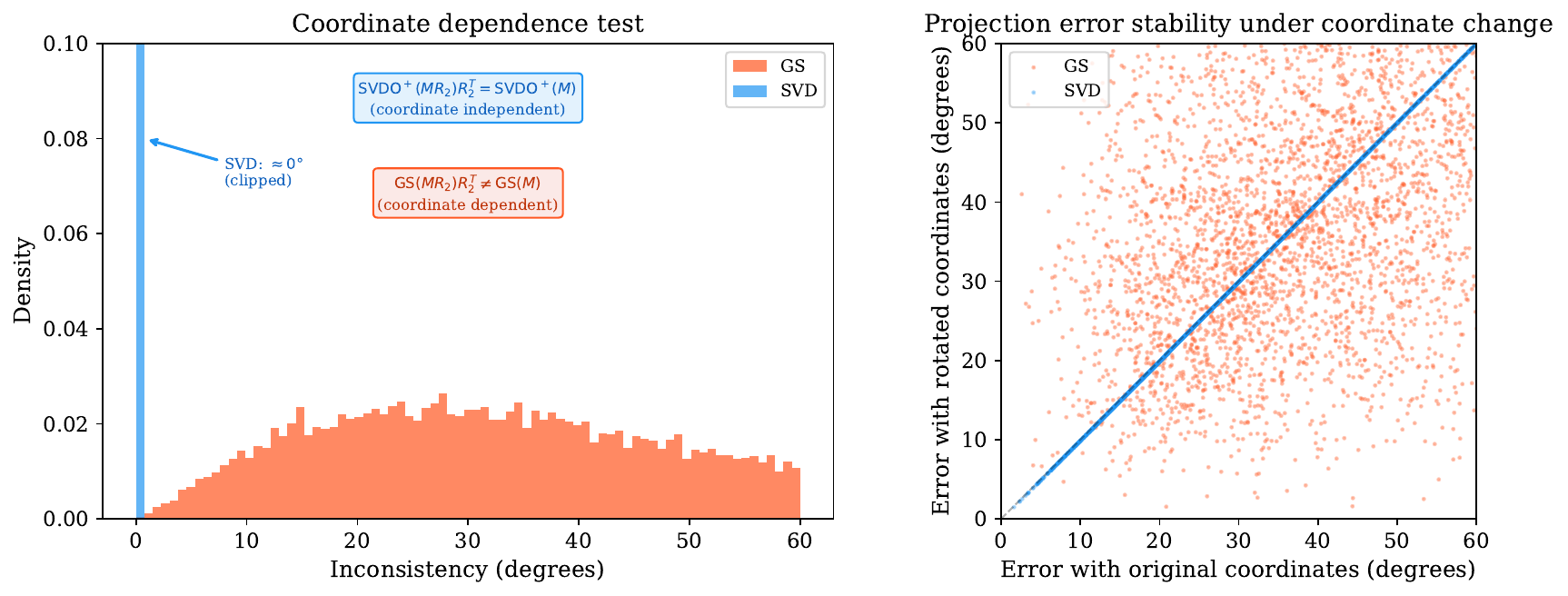}
    \caption{Coordinate dependence test. For random $M$ and $R_2 \in \SO(3)$, we compare $g(M)$ with $g(MR_2)R_2^\top$: for a coordinate-independent projector, these should be identical. \textbf{Left:} SVD produces zero inconsistency (spike at $0°$), confirming $\mathrm{SVDO}^+(MR_2)R_2^\top = \mathrm{SVDO}^+(M)$. GS spreads over $10$--$60°$, showing its result depends on the choice of coordinate frame. \textbf{Right:} SVD projection error is identical regardless of coordinates (points on the diagonal), while GS error changes unpredictably under coordinate rotation.}
    \label{fig:coordinate_dependence}
\end{figure*}

\subsection{Gradient Information Loss}

The rank-$3$ Jacobian means that $6$ of $9$ gradient dimensions are projected to zero during SVD backpropagation.  We quantify this loss.

\begin{proposition}[Gradient information retention]\label{prop:info_loss}
For an isotropic random gradient $\nabla_R \mathcal{L} \sim \mathcal{N}(0, \Id_9)$, define the \emph{gradient information retention} (GIR) as $\eta(J) = \mathbb{E}[\norm{J^\top \nabla_R \mathcal{L}}^2] / \mathbb{E}[\norm{\nabla_R \mathcal{L}}^2]$.  Then:
\begin{enumerate}
    \item For SVD-Train near $\SO(3)$ ($s_1, s_2, s_3 \approx 1$): $\eta(J_{\mathrm{SVD}}) = \frac{1}{3}$.  Two-thirds of gradient energy is lost.
    \item For direct 9D regression: $\eta(J_{\mathrm{id}}) = 1$.  All gradient energy is retained.
\end{enumerate}
The lost $6$-dimensional component corresponds to the normal space of $\SO(3)$ at $R$: perturbations that change singular values ($3$D) and symmetric off-diagonal deformations ($3$D).  This normal component carries information about \emph{how far $M$ is from $\SO(3)$}.  SVD discards it, so the optimizer receives no signal pushing $M$ toward orthogonality.  Direct regression retains this signal, explaining why networks trained with MSE naturally converge to near-orthogonal outputs~\citep{gu2024prom}.
\end{proposition}

\begin{proof}
Since $\mathbb{E}[\norm{J^\top \mathbf{g}}^2] = \tr(J J^\top)$ for $\mathbf{g} \sim \mathcal{N}(0, \Id_9)$, and $J_{\mathrm{SVD}}$ has nonzero singular values $\{2/(s_i+s_j)\}_{i<j}$, we get $\tr(J_{\mathrm{SVD}} J_{\mathrm{SVD}}^\top) = \sum_{i<j} 4/(s_i+s_j)^2$.  Near $\SO(3)$, each term equals $1$, giving $\eta = 3/9 = 1/3$.
\end{proof}

\section{Gram-Schmidt Gradient Asymmetry}
\label{sec:gs_analysis}

SVD should be avoided during training.  What about Gram-Schmidt?  The 6D representation~\citep{zhou2019continuity} uses GS instead, but GS introduces a different problem: \emph{gradient asymmetry}.

\citet{gu2024prom} identified gradient coupling and explosion for GS (their Eq.~5, Section~3.1, Appendix~C.1). We formalize the asymmetric Jacobian structure, specifically the one-directional coupling (Part~2) and condition number bound (Part~4), explaining why 9D is preferable to 6D even without orthogonalization.

\begin{theorem}[Gram-Schmidt Jacobian asymmetry]\label{thm:gs_asymmetry}
Let $g_{\mathrm{GS}}: \R^6 \to \SO(3)$ be the Gram-Schmidt orthogonalization defined in \eqref{eq:gs}, mapping $(\mathbf{t}_1', \mathbf{t}_2') \mapsto (\mathbf{r}_1', \mathbf{r}_2', \mathbf{r}_3')$. Let $J_{\mathrm{GS}} = \pderiv{\mathrm{vec}(R)}{\mathrm{vec}(\mathbf{t}_1', \mathbf{t}_2')} \in \R^{9 \times 6}$. Then:

\begin{enumerate}
    \item \textbf{Column 1 is self-contained:} $\pderiv{\mathbf{r}_1'}{\mathbf{t}_1'}$ depends only on $\mathbf{t}_1'$, not $\mathbf{t}_2'$. Its singular values are $\{1/\norm{\mathbf{t}_1'}, 1/\norm{\mathbf{t}_1'}, 0\}$ (rank 2 due to radial degeneracy).

    \item \textbf{Column 2 depends on Column 1:} $\pderiv{\mathbf{r}_2'}{\mathbf{t}_1'}$ is generically nonzero while $\pderiv{\mathbf{r}_1'}{\mathbf{t}_2'} = 0$. This creates a one-directional coupling: errors in $\mathbf{r}_2'$ affect updates to $\mathbf{t}_1'$, but errors in $\mathbf{r}_1'$ never affect updates to $\mathbf{t}_2'$.

    \item \textbf{Column 3 has no dedicated parameters:} Since $\mathbf{r}_3' = \mathbf{r}_1' \times \mathbf{r}_2'$, the gradient $\pderiv{\mathbf{r}_3'}{\mathbf{t}_k'}$ is entirely mediated through $\mathbf{r}_1'$ and $\mathbf{r}_2'$, compounding the distortion from two normalization layers.

    \item \textbf{Condition number diverges:} As $\mathbf{t}_1'$ and $\mathbf{t}_2'$ become parallel (i.e., $\mathbf{r}_2'' = \mathbf{t}_2' - (\mathbf{r}_1' \cdot \mathbf{t}_2')\mathbf{r}_1' \to 0$), the condition number of the Jacobian restricted to the column space satisfies:
    \begin{equation}
        \kappa(J_{\mathrm{GS}}) \geq \frac{\norm{\mathbf{t}_1'}}{\norm{\mathbf{r}_2''}}.
    \end{equation}
\end{enumerate}
\end{theorem}

The proof follows from standard normalization Jacobians and the chain rule through the cross product; see \Cref{app:gs_proof} for details.

\Cref{thm:gs_asymmetry} reveals a strict gradient hierarchy: $\mathbf{t}_1'$ parameters receive signal from all three rotation columns, while $\mathbf{t}_2'$ parameters are blind to column~1 errors.  This asymmetry also manifests at inference, where GS produces monotonically increasing per-column projection error (\Cref{fig:column_error}).

\begin{figure*}[t]
    \centering
    \includegraphics[width=\textwidth]{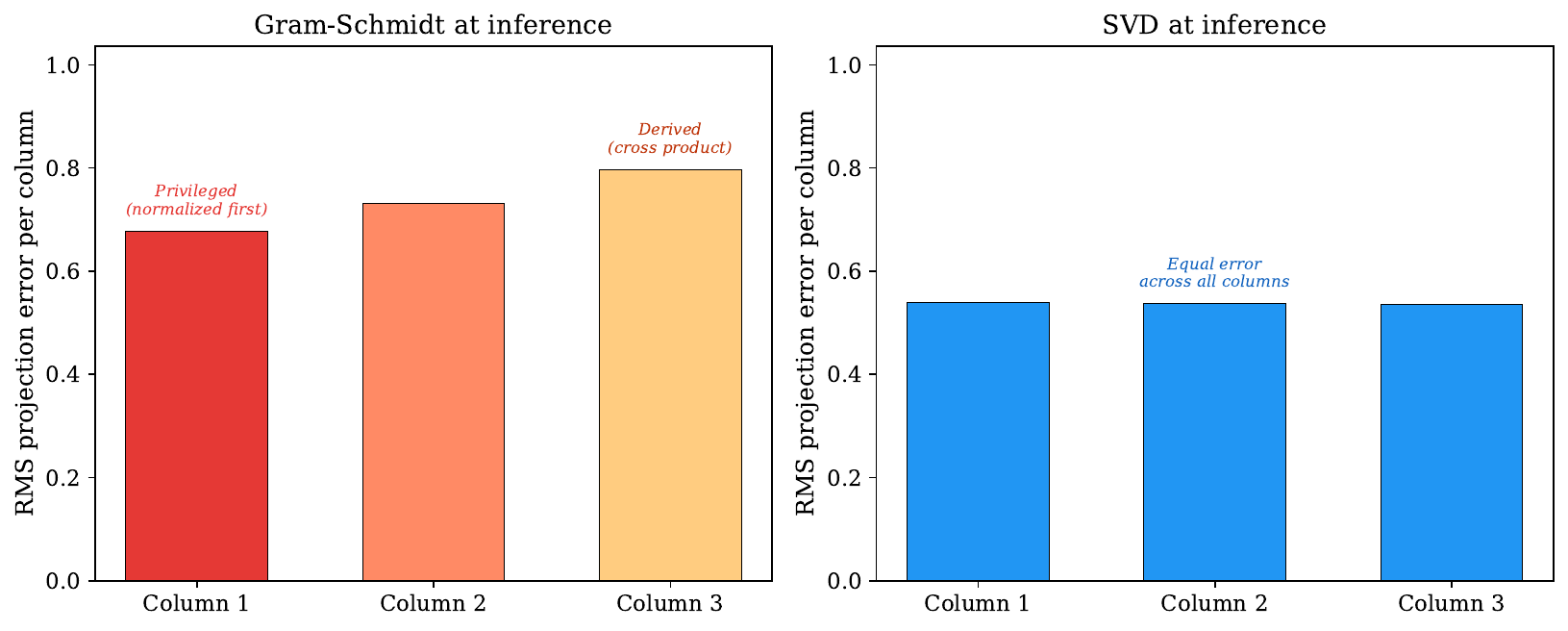}
    \caption{Per-column RMS projection error at inference ($\sigma = 0.5$, $50{,}000$ samples). \textbf{Left:} Gram-Schmidt produces increasing error from column 1 (privileged, normalized first) to column 3 (derived via cross product), reflecting its sequential, greedy forward-pass structure. \textbf{Right:} SVD distributes error equally across all three columns, a consequence of its coordinate-independence. SVD also achieves lower absolute error per column.}
    \label{fig:column_error}
\end{figure*}

\section{Direct 9D Regression and the Principled Synthesis}
\label{sec:synthesis}

Removing orthogonalization gives $J_{\mathrm{id}} = \Id_9$ with $\kappa = 1$ and no cross-coupling.\label{thm:identity_jacobian}  \Cref{tab:all_representations} summarizes the gradient properties of all common rotation representations.

\begin{table*}[t]
\caption{Comparison of rotation representations for gradient-based optimization.  ``Continuous $g$'' refers to whether the mapping from $\SO(3)$ to the representation space is continuous~\citep{zhou2019continuity}.  ``Training $\kappa$'' is the condition number of the representation Jacobian during training.  ``Inference proj.'' indicates the projection method applied at inference to obtain a valid rotation.  Representations above the mid-rule suffer from topological obstructions (discontinuity and/or double cover); those below are continuous but differ in gradient properties.}
\label{tab:all_representations}
\centering
\resizebox{\textwidth}{!}{%
\begin{tabular}{lcccccccc}
\toprule
\textbf{Representation} & \textbf{Dim} & \textbf{Cont.\ $g$} & \textbf{Double} & \textbf{Jacobian} & \textbf{Training} & \textbf{Gradient} & \textbf{Null} & \textbf{Inference} \\
 & & & \textbf{cover} & \textbf{rank} & $\kappa$ & \textbf{blow-up} & \textbf{space} & \textbf{proj.} \\
\midrule
Euler angles & 3 & No & No$^*$ & 3$^\dagger$ & $\to \infty$$^\dagger$ & Gimbal lock & 0$^\dagger$ & None \\
Exp.\ coordinates & 3 & No & Yes & 3$^\dagger$ & $\to \infty$$^\dagger$ & At $\norm{\omega} = 2\pi$ & 0$^\dagger$ & None \\
Axis-angle & 4 & No & Yes & 3 & $\to \infty$$^\dagger$ & At $\theta = \pi$ & 1 & Normalize \\
Quaternion & 4 & No & Yes & 3 & $= 1$ & Never$^\ddagger$ & 1 & Normalize \\
\midrule
$\R^6$ + GS~\citep{zhou2019continuity} & 6 & Yes & No & $\leq 5$ & $\geq \frac{\norm{\mathbf{t}_1'}}{\norm{\mathbf{r}_2''}}$ & $O(1/\norm{\mathbf{r}_2''})$ & $\geq 1$ & GS \\
$\R^9$ + SVD-Train~\citep{levinson2020analysis} & 9 & Yes & No & 3 & $\frac{s_1+s_2}{s_2+s_3}$ & $\frac{2}{s_2+s_3}$ & 6 & SVD \\
$\R^9$ + SVD-Inference~\citep{gu2024prom} & 9 & Yes & No & 9 & $= 1$ & Never & 0 & SVD \\
\bottomrule
\end{tabular}%
}

\vspace{0.5em}
\raggedright
{\scriptsize $^*$Multiple representations exist but not a true double cover. $^\dagger$Generically; degenerates at singularities. $^\ddagger$Normalization $q/\norm{q}$ has $\kappa = 1$, but the double cover creates a \emph{topological} discontinuity: the target function $g \circ h_{\mathrm{true}}$ is discontinuous, which harms learning independently of gradient quality~\citep{zhou2019continuity}.}
\end{table*}

\Cref{tab:all_representations} shows that rotation representations fail for two different reasons.  Low-dimensional representations (Euler angles, quaternions) suffer from \emph{topological} obstructions: the mapping $g: \SO(3) \to \R^d$ is necessarily discontinuous for $d \leq 4$~\citep{stuelpnagel1964parameterization,zhou2019continuity}, making the target function discontinuous regardless of gradient quality.  High-dimensional continuous representations ($\geq 5$D) avoid topological obstructions but differ in gradient properties: GS (6D) has asymmetric gradients (\Cref{thm:gs_asymmetry}); SVD-Train (9D) has $\kappa = (s_1+s_2)/(s_2+s_3)$ with rank-3 information loss (\Cref{thm:svd_gradient_bound}, \Cref{prop:info_loss}); direct 9D regression achieves $\kappa = 1$ with full-rank gradients and no coupling.

\paragraph{When are quaternions preferable?}
Quaternion normalization shares $\kappa = 1$ with direct 9D regression (\Cref{tab:all_representations}), and quaternions require only 4 output dimensions versus 9. Their sole theoretical disadvantage is the double cover ($q \equiv -q$), which forces either a non-smooth loss ($\min(\norm{q-q^*}^2, \norm{q+q^*}^2)$) or a hemisphere constraint with its own discontinuity.  However, this discontinuity only matters when the data distribution includes rotations near the 180° boundary.  When the rotation distribution is concentrated (e.g., small perturbations around a reference pose, or a narrow viewpoint range), the topological obstruction falls outside the data support and quaternions can outperform higher-dimensional representations~\citep{geist2024learning,peretroukhin2020smooth}.  \citet{geist2024learning} recommend quaternions with a halfspace map specifically for small-rotation regimes.  The 9D representation is preferable when rotations span a large portion of $\SO(3)$ or the distribution is unknown a priori.

\subsection{Why 9D Over 6D?}

With 6D direct regression, including a third-column loss via $\mathbf{t}_1' \times \mathbf{t}_2'$ reintroduces coupling ($\pderiv{(\mathbf{t}_1' \times \mathbf{t}_2')}{\mathbf{t}_1'} = -[\mathbf{t}_2']_\times \neq 0$).  Dropping it loses supervision on $\mathbf{r}_3^*$.  9D avoids this dilemma: independent supervision for all 9 elements with $\kappa = 1$.

\subsection{Theoretical Justification for 9D + SVD-Inference}

\citet{gu2024prom} empirically found that 9D direct regression with SVD at inference is the best-performing configuration, outperforming both SVD-Train~\citep{levinson2020analysis} and 6D with GS at inference (54.8 vs.\ 55.6 vs.\ 56.7 PA-MPJPE).  Even \citet{levinson2020analysis}'s own experiments show SVD-Inference outperforming SVD-Train on Pascal3D+ (non-uniform viewpoints) and converging faster on ModelNet.  But no theoretical analysis has explained these findings.  Our results do:

\paragraph{Why remove orthogonalization during training?}
\citet{gu2024prom} proved $\psi(B_1) = \infty$ for any non-injective $g$.  \Cref{thm:svd_gradient_bound} goes further for SVD: the distortion is quantifiable, with spectral norm $2/(s_2+s_3)$ and condition number $(s_1+s_2)/(s_2+s_3)$.  Stabilization~\citep{wang2021robust} cannot eliminate it without introducing direction error (\Cref{prop:taylor_comparison}).

\paragraph{Why 9D, not 6D?}
The GS Jacobian is asymmetric (\Cref{thm:gs_asymmetry}): $\mathbf{t}_1'$ parameters get gradient signal from all three rotation columns; $\mathbf{t}_2'$ parameters are blind to column 1 errors.  At inference, GS produces increasing per-column projection error (\Cref{fig:column_error}).  Direct 9D regression has neither problem ($J = \Id_9$, $\kappa = 1$).

\paragraph{Why SVD at inference, not GS?}
SVD is the least-squares optimal projector onto $\SO(3)$, with half the expected error of GS (\Cref{fig:projection_error})~\citep{levinson2020analysis}.  The same global coupling that hurts gradients during training (treating all 9 elements symmetrically) is what makes SVD the better projector: it is coordinate-independent ($\mathrm{SVDO}^+(R_1 M R_2) = R_1 \, \mathrm{SVDO}^+(M) \, R_2$, \Cref{fig:coordinate_dependence}), while GS is equivariant on only one side.  At inference there is no backpropagation, so SVD's gradient pathologies do not apply.

\section{SVD Inference: Error Reduction Guarantee}
\label{sec:inference_bound}

SVD projection at inference can only improve the raw matrix prediction.

\begin{proposition}[SVD projection reduces error]\label{prop:first_order}
Let $M = R^* + \sigma N$ with $R^* \in \SO(3)$, $\frob{N} = 1$, and $\sigma$ small.  Decompose ${R^*}^\top N = A + S$ into antisymmetric ($A \in \mathfrak{so}(3)$, tangent to $\SO(3)$) and symmetric ($S$, normal to $\SO(3)$) parts.  Then:
\begin{equation}
    \frob{\mathrm{SVDO}^+(M) - R^*}^2 = \sigma^2 \frob{A}^2 + O(\sigma^3) = \frob{M - R^*}^2 - \sigma^2 \frob{S}^2 + O(\sigma^3).
\end{equation}
SVD projection removes the normal-space component $S$ of the error, strictly reducing the error whenever $S \neq 0$.
\end{proposition}

\begin{proof}
The polar decomposition of $\Id + \sigma(A + S)$ gives orthogonal factor $\Id + \sigma A + O(\sigma^2)$ and positive-definite factor $\Id + \sigma S + O(\sigma^2)$.  By left-equivariance, $\mathrm{SVDO}^+(M) = R^*(\Id + \sigma A) + O(\sigma^2)$.  The error bound follows from the orthogonality $\langle A, S \rangle_F = 0$ (antisymmetric--symmetric decomposition).
\end{proof}

\begin{corollary}[Factor-of-$3$ error reduction]\label{cor:factor_three}
Under isotropic Gaussian noise ($N_{ij} \sim \mathcal{N}(0,1)$), the expected squared error after SVD projection is $1/3$ of the raw error:
\begin{equation}
    \frac{\mathbb{E}[\frob{\mathrm{SVDO}^+(M) - R^*}^2]}{\mathbb{E}[\frob{M - R^*}^2]} = \frac{3\sigma^2 + O(\sigma^3)}{9\sigma^2} = \frac{1}{3} + O(\sigma).
\end{equation}
This $1/3$ ratio reflects the dimension counting: $3$ of $9$ dimensions are tangential to $\SO(3)$ and survive projection, while $6$ are normal and are removed.  The training MSE loss $\epsilon^2 = \mathbb{E}[\frob{M-R^*}^2]$ therefore serves as a conservative upper bound on inference error: $\mathbb{E}[\frob{\hat{R}-R^*}^2] \leq \epsilon^2$, with typical value $\approx \epsilon^2/3$.
\end{corollary}

In short: MSE training drives $\frob{M - R^*} \to 0$; SVD inference removes the normal-space component (reducing error by $\sim\!3\times$); and the training loss upper-bounds the inference error.

\section{Discussion}
\label{sec:discussion}

\paragraph{Frobenius optimality vs.\ MLE optimality.}
Two distinct notions of SVD optimality should be distinguished.  First, $\mathrm{SVDO}^+$ is the nearest rotation in Frobenius norm (\eqref{eq:svd_optimality}), a geometric fact independent of any noise model.  Second, $\mathrm{SVDO}^+$ is the MLE under isotropic Gaussian noise~\citep{levinson2020analysis}, a statistical claim requiring the noise assumption.  For inference-time projection of a trained network's output, the geometric optimality is what matters: the network's residual error is structured (not isotropic), but SVD still provides the closest rotation in Frobenius norm.

\paragraph{Why not orthogonalize during training at all?}
One might ask: if the network already learns near-rotation matrices, why not apply SVD during training for the ``last mile''? Our analysis shows this is counterproductive. The SVD Jacobian introduces $O(1/\delta)$ gradient amplification (\Cref{thm:svd_gradient_bound}) that is \emph{unnecessary}: the MSE loss on the raw matrix already drives the output toward $\SO(3)$ (since the targets are rotation matrices), and the SVD projection at inference handles any residual deviation optimally (\Cref{cor:factor_three}).

\paragraph{Broader implications of the SVD gradient analysis.}
Our analysis of the $2 \times 2$ system \eqref{eq:svd_system} and its determinant $s_j^2 - s_i^2$ applies beyond rotation estimation. Any differentiable pipeline that backpropagates through SVD faces the same $O(1/\delta)$ gradient amplification.  For example, \citet{choy2020deep} use differentiable SVD in a Weighted Procrustes solver for point cloud registration, where gradients flow through the SVD to update correspondence weights.  In that setting, the covariance matrix fed to SVD is constructed from well-matched point pairs, so singular value gaps are typically large and the pathology is mild.  By contrast, in rotation representation learning the network output can be far from $\SO(3)$ early in training, making the instability practically relevant.

\paragraph{When is GS preferable to SVD at inference?}
We note that \citet{gu2024prom} use Gram-Schmidt rather than SVD at inference for body and hand pose estimation tasks, despite using SVD for rotation estimation tasks.  This may reflect practical considerations such as compatibility with parametric models (SMPL/MANO) or the marginal benefit of SVD when predictions are already near $\SO(3)$.  Our analysis establishes SVD's theoretical superiority as a projector, but the practical gap may be small when the network is well-trained.

\paragraph{Limitations.}
Our main analysis focuses on the Frobenius norm loss.  In \Cref{app:geodesic}, we show that the geodesic loss compounds an additional $O(1/\sin\theta)$ singularity with the SVD pathology, further strengthening the case for direct regression.  Average-case analysis (\Cref{app:expected_condition}) and convergence rate comparisons (\Cref{app:convergence}) are provided in the appendix.

\section{Conclusion}
\label{sec:conclusion}

We have presented a detailed gradient analysis of SVD orthogonalization for rotation estimation, deriving the exact spectrum of the $3 \times 3$ Jacobian: rank $3$ with nonzero singular values $2/(s_i+s_j)$ and condition number $\kappa = (s_1+s_2)/(s_2+s_3)$.  This analysis goes beyond prior work's general non-injectivity arguments by quantifying how the singular value gap controls gradient distortion, and by showing that even state-of-the-art stabilization cannot eliminate this distortion without introducing direction error.

As complementary results, we proved that Gram-Schmidt introduces asymmetric gradient signal across its 6 parameters, while direct 9D regression achieves perfect conditioning with $\kappa = 1$.  Together, these results provide the theoretical foundation for the empirically successful approach of training with direct 9D regression and applying SVD projection only at inference~\citep{gu2024prom}, explaining the mechanism behind what was previously an empirical observation.

\bibliographystyle{plainnat}
\bibliography{references}

\appendix

\section{Proof of \Cref{cor:det_negative}: SVD Jacobian for $\det(M) < 0$}
\label{app:det_negative}

\Cref{thm:svd_gradient_bound} derives the Jacobian spectrum for $\det(M) > 0$, where $\mathrm{SVDO}^+(M) = UV^\top$.
We now analyze the complementary case $\det(M) < 0$, where
\begin{equation}\label{eq:svdo_neg}
    \mathrm{SVDO}^+(M) = U \diag(1,1,-1) V^\top.
\end{equation}
This regime is practically relevant: when the network output $M$ has i.i.d.\ Gaussian entries (as approximately holds at random initialization), $\det(M) < 0$ with probability $1/2$ by symmetry.

\begin{theorem}[SVD Jacobian spectrum, $\det(M) < 0$]\label{thm:svd_neg_det_full}
Let $M \in \R^{3 \times 3}$ with $\det(M) < 0$ and SVD $M = U \Sigma V^\top$ with distinct singular values $s_1 > s_2 > s_3 > 0$.  Let $\Sigma' = \diag(1,1,-1)$ and $R = U\Sigma' V^\top = \mathrm{SVDO}^+(M)$.  The Jacobian $J_{\mathrm{SVD}}^{-} = \pderiv{\mathrm{vec}(R)}{\mathrm{vec}(M)}$ has rank $3$ with a $6$-dimensional null space, and its three nonzero singular values are:
\begin{equation}\label{eq:svd_neg_spectrum}
    \sigma(J_{\mathrm{SVD}}^{-}) = \left\{\frac{2}{s_1+s_2},\; \frac{2}{s_1+s_3},\; \frac{2}{s_2+s_3},\; 0,\; 0,\; 0,\; 0,\; 0,\; 0\right\},
\end{equation}
which are \textbf{identical} to the $\det(M) > 0$ spectrum in \Cref{thm:svd_gradient_bound}.  Consequently, the spectral norm and condition number are:
\begin{equation}\label{eq:svd_neg_bounds}
    \norm{J_{\mathrm{SVD}}^{-}}_2 = \frac{2}{s_2 + s_3}, \qquad
    \kappa(J_{\mathrm{SVD}}^{-}) = \frac{s_1 + s_2}{s_2 + s_3}.
\end{equation}
However, the sign flip changes which input subspace is active: for pairs $(i,3)$ involving the flipped singular value, the \emph{symmetric} off-diagonal component of $P = U^\top \mathrm{d}M\, V$ drives the output, rather than the antisymmetric component as in the $\det(M) > 0$ case.
\end{theorem}

\begin{proof}
The differential of $R = U \Sigma' V^\top$ is
\begin{equation}
    \mathrm{d}R = \mathrm{d}U \, \Sigma' V^\top + U \Sigma' \, \mathrm{d}V^\top = U(A\Sigma' - \Sigma'\Omega)V^\top,
\end{equation}
where $A = U^\top \mathrm{d}U$ and $\Omega = V^\top \mathrm{d}V$ are antisymmetric.  Define $\Psi = A\Sigma' - \Sigma'\Omega$, so that $\mathrm{d}R = U\Psi V^\top$ and $\frob{\mathrm{d}R} = \frob{\Psi}$.  The off-diagonal entries are:
\begin{equation}\label{eq:psi_offdiag}
    \Psi_{ij} = A_{ij}\Sigma'_{jj} - \Sigma'_{ii}\Omega_{ij}, \qquad i \neq j.
\end{equation}
Substituting the solutions \eqref{eq:omega_full} for $A_{ij}$ and $\Omega_{ij}$ from the system \eqref{eq:svd_system}:
\begin{equation}\label{eq:psi_general}
    \Psi_{ij} = \frac{(\Sigma'_{jj}\, s_j + \Sigma'_{ii}\, s_i)\,P_{ij} \;-\; (\Sigma'_{jj}\, s_i + \Sigma'_{ii}\, s_j)\,P_{ji}}{s_j^2 - s_i^2},
\end{equation}
where $P = U^\top \mathrm{d}M\, V$.  We analyze each pair $(i,j)$ with $i < j$ according to the sign product $c_{ij} = \Sigma'_{ii}\Sigma'_{jj}$.

\medskip
\noindent\textbf{Pair $(1,2)$:} $\Sigma'_{11} = \Sigma'_{22} = 1$, so $c_{12} = +1$.  Then \eqref{eq:psi_general} gives
\begin{equation}
    \Psi_{12} = \frac{(s_2 + s_1)P_{12} - (s_1 + s_2)P_{21}}{s_2^2 - s_1^2} = \frac{P_{21} - P_{12}}{s_1 + s_2},
\end{equation}
and $\Psi_{21} = -\Psi_{12}$ by the antisymmetry of $\Psi$ (which follows from $R^\top R = \Id$).  This depends only on the antisymmetric combination $\alpha_{12} = (P_{21} - P_{12})/\sqrt{2}$, with gain $\sqrt{2\cdot\Psi_{12}^2 + 2\cdot\Psi_{21}^2}/\sqrt{2\alpha_{12}^2} = 2/(s_1+s_2)$---identical to the $\det(M) > 0$ case.

\medskip
\noindent\textbf{Pair $(1,3)$:} $\Sigma'_{11} = 1$, $\Sigma'_{33} = -1$, so $c_{13} = -1$.  Then
\begin{align}
    \Psi_{13} &= \frac{(-s_3 + s_1)P_{13} - (-s_1 + s_3)P_{31}}{s_3^2 - s_1^2} = \frac{(s_1 - s_3)(P_{13} + P_{31})}{-(s_1 - s_3)(s_1 + s_3)} = \frac{-(P_{13} + P_{31})}{s_1 + s_3}.
\end{align}
This depends on the \emph{symmetric} combination $\beta_{13} = (P_{13} + P_{31})/\sqrt{2}$, with gain $2/(s_1+s_3)$.

\medskip
\noindent\textbf{Pair $(2,3)$:} $\Sigma'_{22} = 1$, $\Sigma'_{33} = -1$, so $c_{23} = -1$.  By the same calculation:
\begin{equation}
    \Psi_{23} = \frac{-(P_{23} + P_{32})}{s_2 + s_3},
\end{equation}
with gain $2/(s_2 + s_3)$ on the symmetric combination $\beta_{23}$.

\medskip
\noindent\textbf{Null space structure.}  The $9$ entries of $P$ decompose into orthogonal subspaces:
\begin{enumerate}
    \item \textbf{Diagonal} $P_{ii}$ ($3$ dimensions): $\Psi_{ii} = 0$.  Null space.
    \item \textbf{Pair $(1,2)$, symmetric component} $(P_{12}+P_{21})/\sqrt{2}$: maps to $0$.  Null space.
    \item \textbf{Pair $(1,2)$, antisymmetric component} $(P_{21}-P_{12})/\sqrt{2}$: maps with gain $2/(s_1+s_2)$.
    \item \textbf{Pair $(1,3)$, antisymmetric component} $(P_{31}-P_{13})/\sqrt{2}$: maps to $0$.  Null space.
    \item \textbf{Pair $(1,3)$, symmetric component} $(P_{13}+P_{31})/\sqrt{2}$: maps with gain $2/(s_1+s_3)$.
    \item \textbf{Pair $(2,3)$, antisymmetric component} $(P_{32}-P_{23})/\sqrt{2}$: maps to $0$.  Null space.
    \item \textbf{Pair $(2,3)$, symmetric component} $(P_{23}+P_{32})/\sqrt{2}$: maps with gain $2/(s_2+s_3)$.
\end{enumerate}
The null space has dimension $3 + 3 = 6$, confirming rank $3$.  The three active subspaces are orthogonal and map to orthogonal outputs (distinct off-diagonal pairs of the antisymmetric matrix $\Psi$), so the three nonzero singular values of $J_{\mathrm{SVD}}^-$ are $\{2/(s_1+s_2),\, 2/(s_1+s_3),\, 2/(s_2+s_3)\}$.
\end{proof}

\begin{remark}[The spectrum is sign-invariant]\label{rem:sign_invariance}
The sign flip $\Sigma' = \diag(1,1,-1)$ changes \emph{which} subspace of perturbations drives the rotation differential---symmetric off-diagonal for pairs with opposite signs, antisymmetric for pairs with equal signs---but does not change the denominators $s_i + s_j$.  This is because the denominators arise from the $2\times 2$ system \eqref{eq:svd_system}, whose determinant $s_j^2 - s_i^2$ depends only on the magnitudes of the singular values, not on $\Sigma'$.  The Jacobian spectrum of $\mathrm{SVDO}^+$ is therefore completely determined by $s_1, s_2, s_3$ regardless of the sign of $\det(M)$.
\end{remark}

\begin{remark}[Reconciling with the backward pass formula]\label{rem:backward_reconcile}
The backward pass formula \eqref{eq:svd_grad} for the $\det(M) < 0$ case is sometimes written with effective denominators $s_i - s_3$ for pairs involving the third singular value, obtained by ``replacing $s_3$ with $-s_3$'':
\begin{equation}\label{eq:svd_grad_neg_form}
    \tilde{Z}_{i3} = \frac{-\tilde{X}_{i3}}{s_i - s_3}, \qquad i \in \{1,2\},
\end{equation}
where $\tilde{X}$ is the loss gradient matrix rotated into the SVD frame with $\Sigma'$ absorbed.  This appears to diverge as $s_2 \to s_3$, but there is no contradiction with \Cref{thm:svd_neg_det_full}: the matrix $\tilde{X}$ differs from $X$ by a factor of $\Sigma'$, and the component $\tilde{X}_{i3}$ vanishes proportionally to $s_i - s_3$ when $s_2 \to s_3$, so the product $\tilde{Z}_{i3}$ remains bounded.  The Jacobian singular values---which are basis-independent---confirm that no additional divergence arises from $\det(M) < 0$.

More precisely, \eqref{eq:svd_grad_neg_form} is a coordinate representation of the linear map in a basis where $\Sigma'$ has been absorbed into $\tilde{X}$.  The apparent $1/(s_i - s_3)$ singularity is an artifact of this particular basis choice, not a property of the underlying linear map.  In the natural basis where the Jacobian acts as $P \mapsto \Psi$, the denominators are $s_i + s_3$ as shown in the proof above.
\end{remark}

\begin{remark}[Implications for training stability]
Since the Jacobian spectrum is identical for both signs of $\det(M)$, the gradient pathology during training is fully characterized by \Cref{thm:svd_gradient_bound}: spectral norm $2/(s_2+s_3)$ and condition number $(s_1+s_2)/(s_2+s_3)$, with no additional instability from the $\det(M) < 0$ regime.  The sign of $\det(M)$ does affect which perturbation directions are amplified (symmetric vs.\ antisymmetric off-diagonal), which may cause discrete jumps in the gradient direction when $\det(M)$ crosses zero during training.  However, this is a direction change, not a magnitude change: the gradient norm bounds from \Cref{thm:svd_gradient_bound} apply uniformly.
\end{remark}

\section{Expected Condition Number Under Gaussian Noise}
\label{app:expected_condition}

\Cref{thm:svd_gradient_bound} shows that the SVD Jacobian condition number is $\kappa = (s_1 + s_2)/(s_2 + s_3)$, but does not address a natural follow-up question: \emph{how large is $\kappa$ in expectation} when a network's prediction is a noisy perturbation of the target rotation?  We now answer this using random matrix theory, providing a closed-form approximation and numerical verification.

\subsection{Setup and Singular Value Asymptotics}

We model the network output as
\begin{equation}\label{eq:noise_model}
    M = R + \sigma N, \qquad R \in \SO(3), \quad N_{ij} \overset{\mathrm{i.i.d.}}{\sim} \mathcal{N}(0,1),
\end{equation}
where $\sigma > 0$ controls the noise magnitude.  Since $\kappa$ is invariant under left and right multiplication by orthogonal matrices (the singular values of $M$ are invariant under such transformations), we may take $R = \Id$ without loss of generality, giving $M = \Id + \sigma N$.

The singular values of $M$ are the square roots of the eigenvalues of
\begin{equation}
    M^\top M = \Id + \sigma(N^\top + N) + \sigma^2 N^\top N.
\end{equation}
For small $\sigma$, the dominant perturbation is the symmetric matrix $\sigma(N^\top + N)/2 = \sigma W$, where $W = (N + N^\top)/2$ is a $3 \times 3$ Wigner matrix (GOE, up to scaling) with independent entries: $W_{ii} \sim \mathcal{N}(0,1)$ on the diagonal and $W_{ij} \sim \mathcal{N}(0, 1/2)$ for $i < j$.

\begin{proposition}[Expected condition number of the SVD Jacobian]\label{prop:expected_kappa}
Let $M = \Id + \sigma N$ with $N$ having i.i.d.\ $\mathcal{N}(0,1)$ entries, and let $s_1 \geq s_2 \geq s_3$ be the singular values of $M$.  Assume $\det(M) > 0$ (which holds with high probability for small $\sigma$).  Define
\begin{equation}
    c_3 \;:=\; \frac{3}{2}\sqrt{\frac{3}{\pi}} \;\approx\; 1.466,
\end{equation}
the expected largest eigenvalue of the $3 \times 3$ GOE with our scaling.  Then, to leading order in $\sigma$:
\begin{equation}\label{eq:expected_kappa}
    \mathbb{E}[\kappa] \;\approx\; \frac{2 + \sigma c_3}{2 - \sigma c_3},
\end{equation}
which satisfies $\mathbb{E}[\kappa] = 1 + \sigma c_3 + O(\sigma^2)$ for small $\sigma$, and diverges as $\sigma \to 2/c_3 = \frac{4}{3}\sqrt{\pi/3} \approx 1.364$.
\end{proposition}

\begin{proof}
We proceed in three steps: (i) reduce to the eigenvalue problem for the symmetric part of $N$, (ii) compute the expected ordered eigenvalues exactly, and (iii) substitute into the condition number formula.

\paragraph{Step 1: First-order perturbation of singular values.}
Write $M = \Id + \sigma N$.  Since $\Id$ has all singular values equal to $1$ (a maximally degenerate case), standard matrix perturbation theory gives that, to first order in $\sigma$, the singular values of $M$ are determined by the eigenvalues of the symmetric part of the perturbation.  Specifically, let $N = W + A$ where $W = (N + N^\top)/2$ is symmetric and $A = (N - N^\top)/2$ is antisymmetric.  The antisymmetric part $A$ affects singular values only at second order, since
\begin{equation}
    M^\top M = \Id + 2\sigma W + \sigma^2(W^2 + A^2) + O(\sigma^2),
\end{equation}
and taking square roots, $s_k = \sqrt{1 + 2\sigma\lambda_k(W) + O(\sigma^2)} = 1 + \sigma\lambda_k(W) + O(\sigma^2)$.
Therefore:
\begin{equation}\label{eq:sv_approx}
    s_k \;\approx\; 1 + \sigma \lambda_k(W), \qquad k = 1, 2, 3,
\end{equation}
where $\lambda_1(W) \geq \lambda_2(W) \geq \lambda_3(W)$ are the ordered eigenvalues of $W$.

\paragraph{Step 2: Expected eigenvalues of the $3 \times 3$ GOE.}
The matrix $W = (N + N^\top)/2$ has joint eigenvalue density (for ordered $\lambda_1 \geq \lambda_2 \geq \lambda_3$):
\begin{equation}
    p(\lambda_1, \lambda_2, \lambda_3) \;=\; \frac{1}{Z}\,(\lambda_1 - \lambda_2)(\lambda_1 - \lambda_3)(\lambda_2 - \lambda_3) \cdot \exp\!\Bigl(-\frac{1}{2}\sum_k \lambda_k^2\Bigr),
\end{equation}
where $Z$ is a normalization constant.  By symmetry, $\mathbb{E}[\tr W] = 0$ implies $\mathbb{E}[\lambda_1] + \mathbb{E}[\lambda_2] + \mathbb{E}[\lambda_3] = 0$.  The distribution is also symmetric under $\lambda_k \to -\lambda_k$ (with reversal of ordering), giving $\mathbb{E}[\lambda_1] = -\mathbb{E}[\lambda_3]$ and $\mathbb{E}[\lambda_2] = 0$.

The expected maximum eigenvalue can be computed by integrating against the marginal density.  For $n = 3$, this integral evaluates to an exact closed form:
\begin{equation}\label{eq:expected_eig}
    \mathbb{E}[\lambda_1(W)] = \frac{3}{2}\sqrt{\frac{3}{\pi}} = c_3 \approx 1.466, \qquad \mathbb{E}[\lambda_2(W)] = 0, \qquad \mathbb{E}[\lambda_3(W)] = -c_3.
\end{equation}
This can be verified numerically: sampling $10^6$ instances of $3 \times 3$ matrices $W = (N + N^\top)/2$ yields $\mathbb{E}[\lambda_1] \approx 1.466$, confirming~\eqref{eq:expected_eig}.

\paragraph{Step 3: Condition number.}
Substituting~\eqref{eq:sv_approx} and~\eqref{eq:expected_eig} into the condition number formula $\kappa = (s_1 + s_2)/(s_2 + s_3)$ from \Cref{thm:svd_gradient_bound}:
\begin{align}
    s_1 + s_2 &\;\approx\; (1 + \sigma c_3) + 1 = 2 + \sigma c_3, \label{eq:num}\\
    s_2 + s_3 &\;\approx\; 1 + (1 - \sigma c_3) = 2 - \sigma c_3. \label{eq:den}
\end{align}
Therefore:
\begin{equation}\label{eq:kappa_formula}
    \kappa \;\approx\; \frac{2 + \sigma c_3}{2 - \sigma c_3}.
\end{equation}
The approximation $\mathbb{E}[\kappa] \approx \kappa(\mathbb{E}[s_1], \mathbb{E}[s_2], \mathbb{E}[s_3])$ is valid to leading order in $\sigma$, since $\kappa$ is a smooth function of the singular values and their fluctuations are $O(\sigma)$.

\paragraph{Asymptotic behavior.}
For small $\sigma$, a Taylor expansion of~\eqref{eq:kappa_formula} gives:
\begin{equation}\label{eq:kappa_small_sigma}
    \kappa \;\approx\; 1 + \sigma c_3 + \frac{\sigma^2 c_3^2}{2} + O(\sigma^3).
\end{equation}
As $\sigma$ increases, the denominator $2 - \sigma c_3$ approaches zero, and $\kappa$ diverges at $\sigma^* = 2/c_3 = \frac{4}{3}\sqrt{\pi/3} \approx 1.364$.  Physically, this corresponds to $\mathbb{E}[s_3] \to 0$, i.e., the noisy matrix $M$ becomes rank-deficient in expectation, at which point the SVD projection becomes ill-defined.
\end{proof}

\begin{remark}[Comparison with direct regression]
Under the same noise model, the direct regression Jacobian is $J_{\mathrm{id}} = \Id_9$ with $\kappa = 1$ for \emph{all} $\sigma$ (\Cref{thm:identity_jacobian}).  The ratio of condition numbers,
\begin{equation}
    \frac{\kappa_{\mathrm{SVD}}}{\kappa_{\mathrm{direct}}} = \frac{2 + \sigma c_3}{2 - \sigma c_3},
\end{equation}
quantifies the conditioning penalty paid by SVD-Train.  Even at moderate noise levels ($\sigma = 0.3$, typical of early-to-mid training), this gives $\kappa \approx 1.57$, meaning the SVD Jacobian's largest singular value is $1.57\times$ its smallest---a nontrivial anisotropy that interacts poorly with isotropic optimizers.
\end{remark}

\begin{remark}[Validity of the first-order approximation]\label{rem:validity}
The approximation~\eqref{eq:sv_approx} is accurate when $\sigma \ll 1$, where the second-order term $\sigma^2 N^\top N$ is negligible compared to $2\sigma W$.  As $\sigma$ grows toward $1$, the second-order corrections---which tend to push all singular values upward (since $N^\top N$ is positive semidefinite with expected trace $3$)---become relevant.  This makes the actual condition number grow more slowly than the first-order formula predicts, because the upward shift in all singular values partially compensates the spread.  Nevertheless, the formula~\eqref{eq:expected_kappa} captures the qualitative monotonic growth and is quantitatively accurate for $\sigma \lesssim 0.3$ (relative error $< 1\%$), as confirmed below.
\end{remark}

\subsection{Numerical Verification}

To validate \Cref{prop:expected_kappa}, we sample $M = \Id + \sigma N$ with $50{,}000$ samples per noise level, compute the SVD of each sample, and evaluate the empirical mean of $\kappa = (s_1 + s_2)/(s_2 + s_3)$, restricting to samples with $\det(M) > 0$.

\begin{center}
\begin{tabular}{lccccccc}
\toprule
$\sigma$ & 0.05 & 0.1 & 0.2 & 0.3 & 0.5 & 0.7 & 1.0 \\
\midrule
Formula~\eqref{eq:expected_kappa} & 1.076 & 1.158 & 1.344 & 1.564 & 2.157 & 3.107 & 6.487 \\
Empirical $\mathbb{E}[\kappa]$ & 1.076 & 1.159 & 1.348 & 1.572 & 1.947 & 2.160 & 2.320 \\
Relative error & ${<}0.1\%$ & ${<}0.1\%$ & $0.3\%$ & $0.5\%$ & $10.8\%$ & $43.8\%$ & --- \\
\bottomrule
\end{tabular}
\end{center}

\noindent The agreement is excellent for $\sigma \leq 0.3$ (relative error $< 1\%$).  For larger $\sigma$, the first-order approximation overestimates $\kappa$ because the $O(\sigma^2)$ correction from $N^\top N$ shifts all singular values upward, keeping $s_3$ further from zero than the linear prediction suggests (see \Cref{rem:validity}).  Additionally, conditioning on $\det(M) > 0$ preferentially excludes samples with very small $s_3$, further reducing $\mathbb{E}[\kappa]$.  Despite the quantitative discrepancy at large $\sigma$, the formula correctly captures the key qualitative features.

This analysis confirms two key points:
\begin{enumerate}
    \item The SVD Jacobian's condition number grows monotonically with noise level, starting at $\kappa = 1$ (perfect conditioning) when $\sigma = 0$ and degrading as the network output deviates from $\SO(3)$.  Even at moderate noise ($\sigma = 0.2$--$0.3$), the condition number reaches $1.35$--$1.57$, creating measurable gradient anisotropy.
    \item The leading-order formula $\kappa \approx 1 + \sigma c_3$ provides a simple rule of thumb: \emph{the conditioning penalty grows linearly with the noise level}, at a rate of approximately $1.47$ per unit $\sigma$.  During early training, when $\sigma$ is effectively large, the conditioning can be substantially worse than the small-$\sigma$ prediction.
\end{enumerate}
In contrast, direct 9D regression maintains $\kappa = 1$ regardless of $\sigma$, providing another quantitative argument for the ``train without orthogonalization, project at inference'' paradigm (\Cref{sec:synthesis}).

\section{Convergence Rate Comparison}
\label{app:convergence}

The spectral analysis in \Cref{thm:svd_gradient_bound} characterizes the SVD Jacobian pointwise. We now translate this into a convergence rate comparison for gradient descent, making precise the cost of routing gradients through SVD orthogonalization.

\begin{proposition}[Convergence rates for gradient descent]\label{prop:convergence_rate}
Let $R^* \in \SO(3)$ be a fixed target rotation and consider gradient descent on the Frobenius loss with step size $\eta > 0$.

\textbf{(A) Direct 9D regression.}
The loss $\mathcal{L}_{\mathrm{dir}}(M) = \frob{M - R^*}^2$ has gradient $\nabla_M \mathcal{L}_{\mathrm{dir}} = 2(M - R^*)$.  The gradient descent update
\begin{equation}\label{eq:gd_direct}
    M_{t+1} = M_t - \eta \nabla_M \mathcal{L}_{\mathrm{dir}} = (1 - 2\eta) M_t + 2\eta R^*
\end{equation}
converges linearly for any $0 < \eta < 1$:
\begin{equation}\label{eq:direct_convergence}
    \mathcal{L}_{\mathrm{dir}}(M_t) = (1 - 2\eta)^{2t}\, \mathcal{L}_{\mathrm{dir}}(M_0).
\end{equation}
The convergence rate $\rho_{\mathrm{dir}} = |1 - 2\eta|$ is independent of $M$ and achieves one-step convergence at $\eta = 1/2$.

\textbf{(B) SVD-Train.}
The loss $\mathcal{L}_{\mathrm{SVD}}(M) = \frob{\mathrm{SVDO}^+(M) - R^*}^2$ has gradient
\begin{equation}
    \nabla_M \mathcal{L}_{\mathrm{SVD}} = J_{\mathrm{SVD}}^\top\, \nabla_R \mathcal{L}_{\mathrm{SVD}} = 2\, J_{\mathrm{SVD}}^\top\, (R - R^*),
\end{equation}
where $R = \mathrm{SVDO}^+(M)$ and $J_{\mathrm{SVD}} = \pderiv{\mathrm{vec}(R)}{\mathrm{vec}(M)}$.
The effective Hessian of $\mathcal{L}_{\mathrm{SVD}}$ with respect to $M$, at a point where $R = R^*$ (i.e., near convergence), is $H = 2\,J_{\mathrm{SVD}}^\top J_{\mathrm{SVD}}$.  This matrix has eigenvalues
\begin{equation}\label{eq:svd_hessian_eigenvalues}
    \lambda_{ij} = \frac{4}{(s_i + s_j)^2}, \quad i < j, \qquad \lambda = 0 \;\text{(multiplicity 6)},
\end{equation}
where $s_1 \geq s_2 \geq s_3 > 0$ are the singular values of $M$.  Along the column space of $J_{\mathrm{SVD}}$, the per-step contraction factor for the component corresponding to the pair $(i,j)$ is
\begin{equation}\label{eq:svd_contraction}
    \rho_{ij} = \abs{1 - \eta \cdot \frac{4}{(s_i + s_j)^2}}.
\end{equation}
The worst-case (slowest) convergence rate is governed by the smallest nonzero eigenvalue of $H$:
\begin{equation}\label{eq:svd_worst_rate}
    \rho_{\mathrm{SVD}}^{\mathrm{worst}} = \abs{1 - \eta \cdot \frac{4}{(s_1 + s_2)^2}},
\end{equation}
and convergence in all directions requires $0 < \eta < (s_2 + s_3)^2 / 2$ to prevent overshooting the fastest direction.
\end{proposition}

\begin{proof}
\textbf{Part (A).}
The error $E_t = M_t - R^*$ satisfies $E_{t+1} = (1 - 2\eta)E_t$ from \eqref{eq:gd_direct}, giving $E_t = (1 - 2\eta)^t E_0$.  Therefore $\mathcal{L}_{\mathrm{dir}}(M_t) = \frob{E_t}^2 = (1 - 2\eta)^{2t} \frob{E_0}^2$.  For $0 < \eta < 1$, we have $|1 - 2\eta| < 1$, ensuring linear convergence. The Hessian is $\nabla^2 \mathcal{L}_{\mathrm{dir}} = 2\Id_9$, confirming that the convergence rate is state-independent.

\textbf{Part (B).}
From \Cref{thm:svd_gradient_bound}, $J_{\mathrm{SVD}}$ has singular values $\{2/(s_i + s_j)\}_{i<j}$ together with six zeros.  Therefore $J_{\mathrm{SVD}}^\top J_{\mathrm{SVD}}$ has eigenvalues $\{4/(s_i+s_j)^2\}_{i<j}$ on the column space, establishing \eqref{eq:svd_hessian_eigenvalues}.

The gradient descent update $M_{t+1} = M_t - \eta \nabla_M \mathcal{L}_{\mathrm{SVD}}$ can be analyzed by projecting onto the eigendirections of $J_{\mathrm{SVD}}^\top J_{\mathrm{SVD}}$.  For a component aligned with the eigenvector corresponding to the pair $(i,j)$, the linearized contraction factor per step is $|1 - \eta \lambda_{ij}|$, yielding \eqref{eq:svd_contraction}.

The fastest direction corresponds to the pair $(2,3)$ with eigenvalue $\lambda_{23} = 4/(s_2+s_3)^2$, and the slowest to $(1,2)$ with eigenvalue $\lambda_{12} = 4/(s_1+s_2)^2$.  To avoid divergence in the fastest direction we need $\eta \lambda_{23} < 2$, i.e., $\eta < (s_2+s_3)^2/2$.  Subject to this constraint, the slowest direction contracts at rate \eqref{eq:svd_worst_rate}.
\end{proof}

\begin{corollary}[Convergence rate ratio]\label{cor:rate_ratio}
At a point $M$ with singular values $s_1 \geq s_2 \geq s_3 > 0$, define the Jacobian condition number $\kappa = (s_1 + s_2)/(s_2 + s_3)$ as in \eqref{eq:svd_jacobian_bound}.

\begin{enumerate}
    \item \textbf{Step-size--matched comparison.} Using the optimal step size for each method (i.e., $\eta_{\mathrm{dir}} = 1/2$ and $\eta_{\mathrm{SVD}} = (s_2+s_3)^2/4$ to equalize the fastest contraction rate), the convergence rate ratio in the slowest direction is:
    \begin{equation}\label{eq:rate_ratio}
        \frac{\rho_{\mathrm{SVD}}^{\mathrm{worst}}}{\rho_{\mathrm{dir}}^{\mathrm{worst}}} = \frac{1 - 1/\kappa^2}{0} = \infty \quad \text{(direct achieves one-step convergence)}.
    \end{equation}
    More meaningfully, with equal step size $\eta$ (small enough for both methods to converge), the number of iterations for SVD-Train to reduce the loss by a factor $\epsilon$ in the slowest direction scales as
    \begin{equation}\label{eq:iteration_ratio}
        \frac{N_{\mathrm{SVD}}}{N_{\mathrm{dir}}} \geq \frac{\log(1 - 2\eta)}{\log\!\left(1 - \frac{4\eta}{(s_1+s_2)^2}\right)} \;\approx\; \frac{(s_1+s_2)^2}{2},
    \end{equation}
    where the approximation holds for small $\eta$.  SVD-Train requires $\sim (s_1+s_2)^2/2$ times more iterations in its slowest direction.

    \item \textbf{Near $\SO(3)$} ($s_1, s_2, s_3 \approx 1$): The eigenvalues of $J_{\mathrm{SVD}}^\top J_{\mathrm{SVD}}$ are all $4/(1+1)^2 = 1$, so the effective Hessian is the identity restricted to the column space.  With step size $\eta$:
    \begin{equation}
        \rho_{\mathrm{SVD}} \approx |1 - \eta|, \qquad \rho_{\mathrm{dir}} = |1 - 2\eta|.
    \end{equation}
    For small $\eta$, $\rho_{\mathrm{SVD}} \approx 1 - \eta$ while $\rho_{\mathrm{dir}} \approx 1 - 2\eta$, so \textbf{SVD-Train converges approximately $2\times$ slower than direct regression even in the best case}, when $M$ is already near $\SO(3)$.

    \item \textbf{Far from $\SO(3)$} ($s_3 \ll 1$, $s_1 \gg 1$): The condition number $\kappa = (s_1 + s_2)/(s_2 + s_3)$ grows, and the slowest eigenvalue $\lambda_{12} = 4/(s_1 + s_2)^2 \ll 1$.  The iteration ratio degrades as:
    \begin{equation}\label{eq:far_ratio}
        \frac{N_{\mathrm{SVD}}}{N_{\mathrm{dir}}} \;\gtrsim\; \frac{(s_1 + s_2)^2}{2} \;\gg\; 1.
    \end{equation}
    Simultaneously, the step size must satisfy $\eta < (s_2 + s_3)^2/2$ to prevent divergence in the fast direction, further constraining the rate. For example, with $s_1 = 3$, $s_2 = 1$, $s_3 = 0.1$, the slowest SVD eigenvalue is $4/(3+1)^2 = 1/4$ while the maximum step size is $(1+0.1)^2/2 \approx 0.6$. Even at the optimal $\eta \approx 0.3$, the slowest contraction rate is $\rho_{12} \approx 1 - 0.3/4 = 0.925$, requiring roughly $\log(0.01)/\log(0.925) \approx 59$ iterations to reduce the slowest component by $100\times$.  Direct regression with $\eta = 0.49$ achieves $\rho = 0.02$ and reaches the same reduction in $\sim 2$ iterations.
\end{enumerate}
\end{corollary}

\begin{proof}
\textbf{Part 1.}  For the same step size $\eta$, the number of iterations to reduce a component by factor $\epsilon$ is $N = \log \epsilon / \log \rho$.  For the direct method, $\rho_{\mathrm{dir}} = |1 - 2\eta|$; for SVD-Train in the slowest direction, $\rho_{12} = |1 - 4\eta/(s_1+s_2)^2|$.  Taking the ratio and applying $\log(1-x) \approx -x$ for small $x$ gives \eqref{eq:iteration_ratio}.

\textbf{Part 2.}  When $s_i \approx 1$ for all $i$, we have $(s_i + s_j) \approx 2$ for all pairs, so $\lambda_{ij} \approx 4/4 = 1$.  The SVD contraction rate becomes $|1 - \eta|$.  For the direct method, $\rho = |1 - 2\eta|$.  The ratio $\log(1-2\eta)/\log(1-\eta) \approx 2$ for small $\eta$, confirming the factor-of-2 slowdown.

\textbf{Part 3.}  The step size constraint $\eta < (s_2+s_3)^2/2$ and the slowest eigenvalue $4/(s_1+s_2)^2$ together determine the achievable convergence rate. The numerical example follows by direct substitution.
\end{proof}

\begin{remark}[Interpretation]
The $2\times$ slowdown near $\SO(3)$ has a clean geometric interpretation: the SVD Jacobian is a rank-3 projector onto the tangent space of $\SO(3)$, which ``sees'' only the 3 antisymmetric degrees of freedom of the 9-dimensional perturbation.  The direct loss, by contrast, sees all 9 degrees of freedom equally, giving it twice the effective curvature per step in the rotation-relevant directions.  Far from $\SO(3)$, the anisotropic scaling $1/(s_i + s_j)$ creates a condition number $\kappa^2$ separation between the fastest and slowest convergence rates, forcing the step size to be small (to control the fast direction) while the slow direction barely moves---the classic ill-conditioning bottleneck.

This analysis also explains why adaptive optimizers (Adam, AdaGrad) partially mitigate SVD-Train's disadvantage: per-parameter learning rates can compensate for the anisotropic eigenvalue spectrum. However, they cannot recover the factor-of-2 loss that persists even when $\kappa = 1$, nor can they address the rank deficiency (6-dimensional null space) of the SVD Jacobian.
\end{remark}

\section{Geodesic Loss and Compounded Singularities}
\label{app:geodesic}

Our main analysis (\Cref{sec:svd_analysis,sec:synthesis}) focuses on the Frobenius loss $\mathcal{L}_{\mathrm{Frob}} = \frob{R - R^*}^2$.
A natural question is whether the \emph{geodesic loss}---the Riemannian distance on $\SO(3)$---might interact differently with the SVD gradient pathology.
We show that the geodesic loss introduces an \emph{additional} singularity that compounds with the SVD Jacobian, making the case for direct regression even stronger.

\subsection{The Geodesic Loss on $\SO(3)$}

\begin{definition}[Geodesic loss]
For $R, R^* \in \SO(3)$, the geodesic distance is the rotation angle between them:
\begin{equation}\label{eq:geodesic_loss}
    \mathcal{L}_{\mathrm{geo}}(R, R^*) = \arccos\!\left(\frac{\tr(R^\top R^*) - 1}{2}\right).
\end{equation}
This equals the magnitude of the rotation vector of $R^\top R^*$, i.e., $\mathcal{L}_{\mathrm{geo}} = \theta$ where $\theta \in [0, \pi]$ is the angle of the relative rotation.
\end{definition}

\subsection{Intrinsic Singularity of the Geodesic Gradient}

\begin{proposition}[Geodesic gradient singularity]\label{prop:geodesic_gradient}
Let $\theta = \mathcal{L}_{\mathrm{geo}}(R, R^*)$ with $\theta \in (0, \pi)$.  The gradient of the geodesic loss with respect to $R$ is:
\begin{equation}\label{eq:geodesic_grad_R}
    \pderiv{\mathcal{L}_{\mathrm{geo}}}{R_{ij}} = -\frac{R^*_{ij}}{2\sin\theta}.
\end{equation}
As $R \to R^*$ (i.e., $\theta \to 0$), $\sin\theta \to 0$ and the gradient diverges: $\norm{\pderiv{\mathcal{L}_{\mathrm{geo}}}{R}}_F = O(1/\theta)$.
\end{proposition}

\begin{proof}
Let $c = (\tr(R^\top R^*) - 1)/2$, so that $\theta = \arccos(c)$.  By the chain rule,
\begin{equation}
    \pderiv{\mathcal{L}_{\mathrm{geo}}}{R_{ij}} = \frac{\mathrm{d}\arccos(c)}{\mathrm{d}c} \cdot \pderiv{c}{R_{ij}} = \frac{-1}{\sqrt{1 - c^2}} \cdot \frac{R^*_{ij}}{2},
\end{equation}
where $\pderiv{c}{R_{ij}} = R^*_{ij}/2$ follows from $c = \frac{1}{2}\sum_{k} R_{ki} R^*_{ki} - \frac{1}{2}$.
Since $c = \cos\theta$, we have $\sqrt{1 - c^2} = \sin\theta$, giving \eqref{eq:geodesic_grad_R}.

For the norm, $\norm{\pderiv{\mathcal{L}_{\mathrm{geo}}}{R}}_F = \frob{R^*}/(2\sin\theta) = \sqrt{3}/(2\sin\theta)$, since $\frob{R^*} = \sqrt{3}$ for any rotation matrix.  As $\theta \to 0$, $\sin\theta \sim \theta$, so the gradient norm grows as $\sqrt{3}/(2\theta)$.
\end{proof}

\begin{remark}[Nature of the singularity]\label{rem:geodesic_singularity}
The $1/\sin\theta$ divergence in \Cref{prop:geodesic_gradient} is an \emph{intrinsic} property of the geodesic loss, independent of any rotation representation.  It arises because $\arccos$ has infinite derivative at $c = 1$ (i.e., $\theta = 0$).  Geometrically, while $\mathcal{L}_{\mathrm{geo}}$ measures the true rotation angle, its gradient in the ambient $\R^{3 \times 3}$ space requires dividing by $\sin\theta$---the radius of the latitude circle on $\SO(3)$ at angle $\theta$ from the identity.  This singularity is absent from the Frobenius loss, whose gradient $\pderiv{\mathcal{L}_{\mathrm{Frob}}}{R} = 2(R - R^*)$ vanishes smoothly as $R \to R^*$.
\end{remark}

\subsection{Compounded Singularities Under SVD-Train}

When SVD orthogonalization is used during training, the predicted rotation is $R = \mathrm{SVDO}^+(M)$ and the training loss is $\mathcal{L}_{\mathrm{geo}}(R, R^*)$.  The chain rule gives:
\begin{equation}\label{eq:geodesic_chain}
    \pderiv{\mathcal{L}_{\mathrm{geo}}}{M} = \pderiv{\mathcal{L}_{\mathrm{geo}}}{R} \cdot J_{\mathrm{SVD}}.
\end{equation}

\begin{proposition}[Compounded singularities]\label{prop:compounded}
For SVD-Train with geodesic loss $\mathcal{L}_{\mathrm{geo}}(\mathrm{SVDO}^+(M), R^*)$, the gradient $\pderiv{\mathcal{L}_{\mathrm{geo}}}{M}$ suffers from two independent sources of divergence:
\begin{enumerate}
    \item \textbf{Geodesic singularity:} from $\pderiv{\mathcal{L}_{\mathrm{geo}}}{R}$, which scales as $O(1/\sin\theta)$ (\Cref{prop:geodesic_gradient}).
    \item \textbf{SVD singularity:} from $J_{\mathrm{SVD}}$, whose spectral norm is $2/(s_2 + s_3)$ (\Cref{thm:svd_gradient_bound}).
\end{enumerate}
These singularities are generically independent: the geodesic singularity occurs when $R \to R^*$ (small rotation error), while the SVD singularity occurs when $s_3 \to 0$ (the predicted matrix is far from $\SO(3)$).  In the worst case, both can be active simultaneously, giving:
\begin{equation}\label{eq:compounded_bound}
    \norm{\pderiv{\mathcal{L}_{\mathrm{geo}}}{M}}_F \leq \norm{\pderiv{\mathcal{L}_{\mathrm{geo}}}{R}}_F \cdot \norm{J_{\mathrm{SVD}}}_2 = \frac{\sqrt{3}}{(s_2 + s_3)\sin\theta}.
\end{equation}
In particular, early in training when the singular value gap is small ($s_3 \approx 0$), the SVD term contributes $O(1/s_3)$; late in training when convergence approaches $\theta \to 0$, the geodesic term contributes $O(1/\theta)$.  These two regimes need not be disjoint: a mini-batch that simultaneously has small $\theta$ (a near-correct prediction) and small $s_3$ (a poorly conditioned matrix) experiences both singularities.
\end{proposition}

\begin{proof}
The bound \eqref{eq:compounded_bound} is a direct consequence of the submultiplicativity of operator norms applied to the chain rule \eqref{eq:geodesic_chain}, combined with \Cref{prop:geodesic_gradient} and \Cref{thm:svd_gradient_bound}.  The independence claim follows from the observation that $\theta = \mathcal{L}_{\mathrm{geo}}(UV^\top, R^*)$ depends on the \emph{singular vectors} of $M$, while $s_3$ is a \emph{singular value} of $M$.  One can construct matrices $M$ with any prescribed combination of $\theta$ and $s_3$ by choosing $U, V$ to set the rotation angle and $\Sigma$ to set the singular value gap independently.
\end{proof}

\subsection{Direct Regression Avoids Both Singularities}

For direct 9D regression, the training loss is $\mathcal{L}_{\mathrm{direct}} = \frob{M - R^*}^2$, with gradient:
\begin{equation}\label{eq:direct_grad}
    \pderiv{\mathcal{L}_{\mathrm{direct}}}{M} = 2(M - R^*).
\end{equation}
This gradient has no singularity of any kind: it vanishes linearly as $M \to R^*$, has no dependence on singular values or angular quantities, and requires no division by $\sin\theta$ or by singular value gaps.

At inference, one applies $\hat{R} = \mathrm{SVDO}^+(M)$ to obtain a proper rotation matrix.  If a geodesic-based evaluation metric is desired, it is computed on the projected output $\hat{R}$---but crucially, no gradient flows through this projection.

\begin{remark}[Avoiding the geodesic singularity during training]\label{rem:avoiding_geodesic}
The key insight is that the geodesic loss singularity is a property of the \emph{loss function}, not the \emph{evaluation metric}.  By training with $\mathcal{L}_{\mathrm{direct}} = \frob{M - R^*}^2$ and evaluating with $\mathcal{L}_{\mathrm{geo}}(\mathrm{SVDO}^+(M), R^*)$, one obtains the geometrically meaningful angular error at test time without ever exposing the training gradient to the $1/\sin\theta$ singularity.  Since $\mathcal{L}_{\mathrm{Frob}}$ and $\mathcal{L}_{\mathrm{geo}}$ are monotonically related for small angles ($\frob{R - R^*}^2 = 2\tr(\Id - R^\top R^*) = 4(1 - \cos\theta) \approx 2\theta^2$ for $\theta \ll 1$), minimizing the Frobenius loss on the raw matrix drives the geodesic error to zero as well.
\end{remark}

\begin{remark}[Comparison of singularity sources]\label{rem:singularity_comparison}
The following table summarizes the singularity structure across training configurations:
\begin{center}
\begin{tabular}{lcc}
\toprule
\textbf{Training configuration} & \textbf{Geodesic singularity} & \textbf{SVD singularity} \\
\midrule
SVD-Train + geodesic loss & $O(1/\sin\theta)$ & $O(1/(s_2+s_3))$ \\
SVD-Train + Frobenius loss & None & $O(1/(s_2+s_3))$ \\
Direct + geodesic loss on $\mathrm{SVDO}^+(M)$ & $O(1/\sin\theta)$ & $O(1/(s_2+s_3))$ \\
Direct + Frobenius loss on $M$ & None & None \\
\bottomrule
\end{tabular}
\end{center}
Only the last configuration---direct 9D regression with Frobenius loss on the raw matrix---is entirely free of gradient singularities.  The third row shows that even with direct regression, if one insists on using geodesic loss on $\mathrm{SVDO}^+(M)$ during training, both singularities reappear through the chain rule.  The correct approach is to decouple the training loss (singularity-free) from the evaluation metric (which may use the geodesic distance, since no gradient is required).
\end{remark}

\section{Proof of \Cref{thm:gs_asymmetry}}
\label{app:gs_proof}

\textbf{Part 1.} The normalization $\mathbf{r}_1' = \mathbf{t}_1'/\norm{\mathbf{t}_1'}$ has Jacobian $\pderiv{\mathbf{r}_1'}{\mathbf{t}_1'} = \frac{1}{\norm{\mathbf{t}_1'}}(\Id_3 - \mathbf{r}_1' {\mathbf{r}_1'}^\top)$, the tangent-plane projector scaled by $1/\norm{\mathbf{t}_1'}$, with eigenvalues $\{1/\norm{\mathbf{t}_1'}, 1/\norm{\mathbf{t}_1'}, 0\}$.

\textbf{Part 2.} $\pderiv{\mathbf{r}_1'}{\mathbf{t}_2'} = 0$ since $\mathbf{r}_1'$ depends only on $\mathbf{t}_1'$.  Conversely, $\mathbf{r}_2'' = \mathbf{t}_2' - (\mathbf{r}_1' \cdot \mathbf{t}_2')\mathbf{r}_1'$ depends on $\mathbf{t}_1'$ through $\mathbf{r}_1'$, so $\pderiv{\mathbf{r}_2'}{\mathbf{t}_1'} \neq 0$ generically.

\textbf{Part 3.} By the chain rule, $\pderiv{\mathbf{r}_3'}{\mathbf{t}_k'} = [\mathbf{r}_2']_\times^\top \pderiv{\mathbf{r}_1'}{\mathbf{t}_k'} + [\mathbf{r}_1']_\times \pderiv{\mathbf{r}_2'}{\mathbf{t}_k'}$, where $[\cdot]_\times$ is the skew-symmetric cross-product matrix.

\textbf{Part 4.} The normalization $\mathbf{r}_2' = \mathbf{r}_2''/\norm{\mathbf{r}_2''}$ contributes a factor $1/\norm{\mathbf{r}_2''}$.  A perturbation $\mathrm{d}\mathbf{t}_2'$ orthogonal to $\mathbf{r}_1'$ achieves $\frob{\mathrm{d}R} = O(1/\norm{\mathbf{r}_2''})$, while $\mathrm{d}\mathbf{t}_1'$ tangent to the unit sphere achieves $\frob{\mathrm{d}R} = O(1/\norm{\mathbf{t}_1'})$, giving $\kappa \geq \norm{\mathbf{t}_1'}/\norm{\mathbf{r}_2''}$.

\end{document}